%% file: main.tex
\documentclass[11pt]{article}

\input{math_commands.tex}
\input{commands}
\usepackage{tgpagella}
\usepackage[utf8]{inputenc} 
\usepackage[T1]{fontenc}    
\usepackage{tcolorbox}
\usepackage{mathtools}
\usepackage{hyperref}
\usepackage{cite}
\usepackage{authblk}
\hypersetup{
    colorlinks=true,
    linkcolor=blue,
    citecolor=magenta,      
    urlcolor=cyan,
    pdfpagemode=FullScreen,
}
\usepackage{url}
\usepackage{subcaption}
\usepackage{graphicx}
\usepackage{lipsum}  
\usepackage{cleveref}
\usepackage{xcolor}         
\usepackage{multirow}

\newcommand{\bb}{\mathbb}

\newcommand{\DS}[1]{{\color{magenta} (DS: #1)}}
\newcommand{\ZK}[1]{{\color{cyan} [{\em Zekai:} #1]}}

\newcommand{\qq}[1]{\textcolor{blue}{\bf [{\em Qing:} #1]}}

\usepackage[top=1in, bottom=1in, left=1in, right=1in]{geometry}

\usepackage{url}

\title{Efficient Compression of Overparameterized Deep Models through Low-Dimensional Learning Dynamics}

\author[1]{Soo Min Kwon\footnote{The first two authors contributed to this work equally. Correspondence to \texttt{kwonsm@umich.edu}.}}
\author[3]{Zekai Zhang$^*$}
\author[1]{Dogyoon Song}
\author[1,2]{Laura Balzano}
\author[1,2]{Qing Qu}

\affil[1]{Department of Electrical Engineering \& Computer Science, University of Michigan}
\affil[2]{Michigan Institute for Data Science, University of Michigan}
\affil[3]{Department of Automation, Tsinghua University}

\date{\today}

\begin{document}

\maketitle

\begin{abstract}

Overparameterized models have proven to be powerful tools for solving various machine learning tasks. However, overparameterization often leads to a substantial increase in computational and memory costs, which in turn requires extensive resources to train.
In this work, we present a novel approach for compressing overparameterized models, developed through studying their learning dynamics. We observe that for many deep models, updates to the weight matrices occur within a low-dimensional invariant subspace. For deep linear models, we demonstrate that their principal components are fitted incrementally within a small subspace, and use these insights to propose a compression algorithm for deep linear networks that involve decreasing the width of their intermediate layers.
We empirically evaluate the effectiveness of our compression technique on matrix recovery problems. 
Remarkably, by using an initialization that exploits the structure of the problem, we observe that our compressed network converges faster than the original network, consistently
yielding smaller recovery errors. 
We substantiate this observation by developing a theory focused on deep matrix factorization. Finally, we empirically demonstrate how our
compressed model has the potential to improve the utility of deep nonlinear models.
Overall, our algorithm improves the training efficiency
by more than $2\times$, without compromising generalization.

\end{abstract}

\tableofcontents

\input{contents/new_intro}

\input{contents/formulation}

\input{contents/compression}
\input{contents/experiments}


{\small 
\bibliographystyle{unsrt}
\bibliography{reference}
}


\clearpage

\appendix
\input{contents/additional_results}
\input{contents/new_proofs}

\end{document}

%% file: math_commands.tex
\usepackage{amsmath,amsfonts,bm}
\usepackage{amsthm, amssymb}
\usepackage{comment}
\usepackage{mathtools}
\usepackage{enumitem}
\setlist{leftmargin=5mm}
\usepackage{algorithm}
\usepackage{algpseudocode}

\DeclareFontFamily{U}{mathx}{}
\DeclareFontShape{U}{mathx}{m}{n}{<-> mathx10}{}
\DeclareSymbolFont{mathx}{U}{mathx}{m}{n}
\DeclareMathAccent{\widecheck}{0}{mathx}{"71}

\newtheorem{manualtheoreminner}{Theorem}

\newtheorem{manualcorollaryinner}{Corollary}

\DeclareMathOperator{\tr}{\mathrm{tr}}

\def\eqref#1{equation~\ref{#1}}

\def\1{\bm{1}}

\newtheorem{theorem}{Theorem}
\newtheorem{proposition}{Proposition}

\newtheorem{corollary}{Corollary}

\newtheorem{assumption}{Assumption}

\newtheorem{lemma}{Lemma}

\DeclareMathAlphabet{\mathsfit}{\encodingdefault}{\sfdefault}{m}{sl}
\SetMathAlphabet{\mathsfit}{bold}{\encodingdefault}{\sfdefault}{bx}{n}

\newcommand{\R}{\mathbb{R}}
\newcommand{\N}{\mathbb{N}}

\newcommand{\rank}{\mathrm{rank}\,}
\theoremstyle{remark}

%% file: commands.tex
\newcommand{\by}{\boldsymbol{y}}
\newcommand{\bA}{\boldsymbol{A}}
\newcommand{\bE}{\boldsymbol{E}}
\newcommand{\bM}{\boldsymbol{M}}
\newcommand{\bW}{\boldsymbol{W}}
\newcommand{\bTheta}{\boldsymbol{\Theta}}

\newcommand{\cA}{\mathcal{A}}
\newcommand{\cP}{\mathcal{P}}

\newcommand{\hbTheta}{\widehat{\bTheta}}

\newcommand{\tbTheta}{\widetilde{\bTheta}}

\newcommand{\sfF}{\mathsf{F}}
\newcommand{\Id}{\mathrm{Id}}
\newcommand{\bMsurr}{\bM^{\text{surr}}}

%% file: contents/new_intro.tex
\section{Introduction}\label{sec:intro}

Overparameterization has proven to be a powerful modeling approach for solving various problems in machine learning and signal processing~\cite{hinton, benefit1, zou2022benefits}. In the literature, it has been observed that overparameterized models yield solutions with superior generalization capabilities. This phenomenon has demonstrated prevalence across a wide range of problems, along with a broad class of model architectures, and has far-reaching implications, including the acceleration of convergence~\cite{arora2018optimization, liu2019accelerating} and the improvement of sample complexity~\cite{arora2019implicit, Sun2022TowardsSO}.
For example, Arora et al.~\cite{arora2019implicit} illustrated the advantages of deep linear models within the context of low-rank matrix recovery. They showed that deeper models promoted low-rank solutions as a function of depth, consequently decreasing the sample complexity required for accurate recovery compared to classical approaches like nuclear norm minimization~\cite{nuc_norm} and the Burer-Monterio factorization~\cite{burer}.
Outside of matrix recovery, there exists an abundance of research showcasing the benefits of overparameterized models under many different settings~\cite{over1, over2, over3, over4}.

Nevertheless, the benefits of overparameterization come with a cost; they require extensive computational resources to train. The number of parameters to estimate rapidly increases with the signal dimension, hindering the use of overparameterized models for large-scale problems (or rather making it infeasible with limited resources). This calls for an urgent need of solutions that reduce computational costs in terms of both memory and speed.


To tackle this issue, researchers have begun to study the learning dynamics of overparameterized models more closely under low-rank settings, seeking opportunities for compression~\cite{lr1,lr2,lr3,lr4}. For instance, recent works have shown that assuming the updates of the weight matrices in nonlinear networks have a low-dimensional structure can significantly reduce the training overhead with only a small tradeoff in accuracy~\cite{hu2021lora, wang2023cuttlefish, horvath2023maestro}. 
However, these results are limited to only specific network architectures under certain settings.
Our first observation is that low-rank weight matrices arise rather universally across a wide range of network architectures when trained with GD.
We demonstrate this in Figure~\ref{fig:plot_svals}, where we plot
the singular values of the weight updates 
for the penultimate weight matrix across various networks, showing that the training largely occurs within a low-dimensional subspace despite all of the parameters being updated.
Although this suggests that it is possible to compress the overparameterized weight matrices, it is not immediately clear how one could tackle this task. In addition, therein lies the question of whether one could construct compressed networks without compromising the performance of their wider counterparts.

\begin{figure}[t!]
     \centering
     \includegraphics[width=\textwidth]{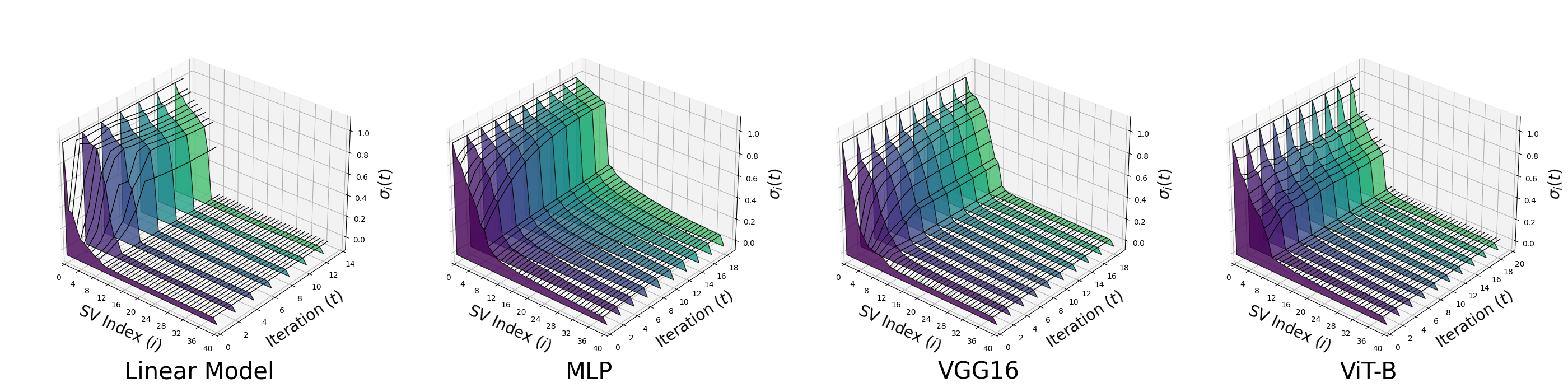}
     \caption{\textbf{Prevalence of low-rank weight updates in various deep networks.} Each plot visualizes the
     singular values of the weight updates from initialization for the penultimate layer weight matrix for different types of (nonlinear) network architectures: including deep linear network (DLN), multi-layer perception (MLP), VGG \cite{DBLP:journals/corr/SimonyanZ14a}, and ViT-B \cite{DBLP:conf/iclr/DosovitskiyB0WZ21} (from left to right). The first two networks (i.e., DLN and MLP) are trained on MNIST, while the latter (i.e., VGG and ViT-B) are trained on CIFAR-10. The result shows a prevalent phenomenon across linear and nonlinear networks -- gradient descent only updates a small portion of the singular values, while the others remain small and almost unchanged. We provide plots for the respective singular vectors and discuss the training details in Appendix~\ref{sec:training_dets}.}
     \label{fig:plot_svals}
\end{figure} 

In this work, we take a step in developing a principled approach for compressing overparameterized models.
By carefully exploring the learning dynamics of DLNs, we unveil several interesting low-rank properties that occur during the training process. Firstly, we demonstrate that the singular subspaces of the DLN are fitted incrementally, similarly to existing work~\cite{li2020towards,jacot2021saddle, chou2023gradient}, but only within a small invariant subspace across a wide range of matrix recovery problems.
Secondly, we leverage this observation to propose a simple, yet highly effective method for compressing DLNs that involve 
decreasing the width of the intermediate layers of the original wide DLN.
The main takeaway of our method is the following:
\begin{tcolorbox}
\begin{center}
    \emph{When properly initialized, the compressed DLN can consistently achieve a lower recovery error than the wide DLN across all iterations of GD, across a wide range of problems.}
\end{center}
\end{tcolorbox}


By capitalizing on the prevalence of incremental learning, we rigorously substantiate
this finding on the deep matrix factorization problem as an illustrative example. We highlight that our approach can provide insights into how one can compress overparameterized weights without increasing recovery error.
Below, we outline some of our key contributions.

\begin{itemize}
    \item \textbf{Fast Convergence and Efficient Training.} By using our compression technique, we demonstrate that our compressed network attains a lower recovery error than the original overparameterized network throughout all iterations of GD. As a result, we are able to achieve (1) faster convergence in less GD iterations and (2) further speed up training by estimating less parameters, all while enjoying the benefits of overparameterized networks.

    \item \textbf{Incremental Learning within an Invariant Subspace.} We empirically demonstrate the prevalence of incremental learning (also commonly referred to as sequential learning), wherein the singular subspaces of a DLN are fitted one at a time (see Figure~\ref{fig:motivation}). We establish that this phenomenon occurs in several canonical matrix recovery problems and leverage these insights to rigorously establish the superiority of our compressed DLN over wide DLNs.

    \item \textbf{Compression of Deep Nonlinear Networks.} We demonstrate how to leverage our findings in deep linear models to accelerate the training of deep nonlinear networks. By overparameterizing the penultimate layer of these deep networks and employing our compression technique, we can reduce memory complexity and training time while achieving similar (or even better) test accuracy.
    

\end{itemize}

\begin{figure}[t!]
    \centering
    \includegraphics[width=0.9\textwidth]{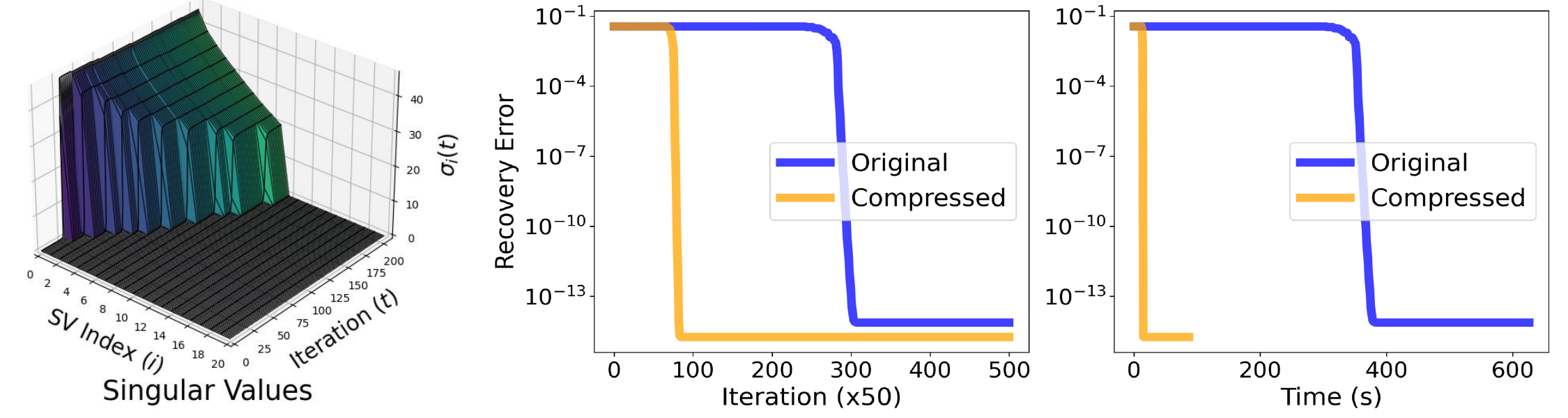}
    \caption{\textbf{Motivating the benefits of our compressed DLN while showcasing the incremental learning phenomenon.}
    Left: plot of the change in singular values of the end-to-end DLN for matrix completion with $r=10$. This shows incremental learning of singual values and implies that we can perform low-rank training within a small subspace without having to overparameterize.  Right: recovery error for the original and compressed DLN across iterations and time, respectively.}
    
    \label{fig:motivation}
\end{figure}
 
\paragraph{Notation and Organization.} We denote vectors with bold lower-case letters (e.g. $\bm{x}$) and matrices with bold upper-case letters (e.g. $\bm{X}$). Given any $n \in \N$, we use $\bm{I}_n$ to denote an identity matrix of size $n$ and $\mathcal{O}^{m \times n} = \{\bm{X}^{m\times n} \,|\, \bm{X}^{\top}\bm{X} = \bm{I}_n\}$ to denote the set of all $m\times n$ orthonormal matrices. We use $[L]$ to denote the set $\{1, 2, \ldots, L\}$.

The rest of the paper is organized as follows. In Section~\ref{sec:compression}, we briefly introduce low-rank matrix recovery problems to motivate our work and propose our compression algorithm. In Section~\ref{sec:theory}, we present our theory, highlighting the benefits of our proposed method. In Section~\ref{sec:experiments}, we empirically demonstrate the effectiveness of our method and the validity of our theory.


%% file: contents/formulation.tex
\section{Efficient Network Compression Method}
\label{sec:compression}

In this section, we first discuss the low-rank matrix recovery problem using DLNs (Section \ref{sec:problem}), and then describe our proposed low-rank network compression method (Section \ref{sec:method}). Additionally, we outline how these ideas can extend to deep nonlinear networks (Section \ref{sec:nonlinear_extension}).

\subsection{A Basic Problem Setup}\label{sec:problem}
We first motivate our study of network compression based upon low-rank matrix recovery problem, where the goal is to estimate a ground truth low-rank matrix $\bm{M}^* \in \R^{d\times d}$ from the measurement $\by = \cA(\bm{M}^*) \in \R^m$, where we assume that the rank $r$ of $\bm{M}^*$ is much smaller than its ambient dimension $(r \ll d)$. Here,  $\cA(\cdot): \R^{d\times d} \to \R^m$ is a linear sensing operator and $m$ denotes the number of measurements.

Moreover, we consider the deep low-rank matrix recovery by modeling the low rank matrix $\bm M^*$ via a DLN parameterized by a sequence of parameters\footnote{In general, we can allow the layer parameter matrices $\bW_l \in \R^{d_l \times d_{l-1}}$ to have arbitrary shapes with any $d_1, \dots, d_{L-1}$ as long as $d_0 = d_L = d$. Here, we assumed square shapes for simplicity in exposition.} $\bTheta = ( \bW_l \in \R^{d \times d})_{l=1}^L$, which can be estimated by solving the following least-squares problem:

\begin{equation}\label{eq:dln_setup}
    \hbTheta \in \underset{\bTheta}{\arg\min}~ \ell_{\cA}(\bTheta; \bM^*)
    \coloneqq \frac{1}{2} \big\| \cA\big( \underbrace{\bW_L \cdot \ldots \cdot \bW_1}_{\eqqcolon \bW_{L:1}} - \bM^* \big) \big\|_{2}^2.
\end{equation}

To obtain the desired low-rank solution, for every iteration $t \geq 0$, we update each weight matrix $\bm W_l$ using GD given by
\begin{align}\label{eqn:gd}
    \bm{W}_l(t) = \bm{W}_l(t-1) - \eta\cdot \nabla_{\bm{W_l}}\ell_{\mathcal{A}}(\bm{\Theta}(t-1)),
\end{align}
$\forall l \in [L]$, where $\eta > 0$ is the learning rate and $\nabla_{\bm{W_l}}\ell_{\mathcal{A}}(\bm{\Theta}(t))$ is the gradient of $\ell_{\mathcal{A}}(\bm{\Theta})$ with respect to the $l$-th weight matrix at the $t$-th GD iterate.
Additionally, we initialize the weights to be orthogonal matrices scaled by a small constant $\epsilon > 0$:
\begin{align}
\label{eq:orth_init}
\bm{W}_l(0)^{\top}\bm{W}_l(0)  = \bm{W}_l(0)\bm{W}_l(0)^{\top} = \epsilon \bm{I},
\end{align}
$\forall l \in [L]$. For the case in which the weight matrices are not square, we can initialize them to be $\epsilon$-scaled semi-orthogonal weight matrices that satisfy
\begin{align}
    \bm{W}_l(0)^{\top}\bm{W}_l(0) = \epsilon \bm{I} \quad \text{or} \quad \bm{W}_l(0)\bm{W}_l(0)^{\top} = \epsilon \bm{I},
\end{align}
which depends on the shape of $\bW_l$.
 It has been proven that orthogonal initialization leads to favorable convergence properties when training DLNs~\cite{yaras2023law, orth1, orth2}, which we adopt throughout this paper for analysis.\footnote{Nonetheless, our experiments show that our method is effective even with small \emph{random} uniform initialization.}

\paragraph{Examples of Deep Matrix Recovery.}
The low-rank matrix recovery problem appears in various machine learning applications, and \emph{deep} low-rank matrix recovery has received much attention recently. 
Among others, here we illustrate three prominent such examples.
\begin{itemize}[leftmargin=*]
    \item 
    \textbf{Deep Matrix Factorization.} 
    When $\mathcal{A}(\cdot)$ is the identity map $\Id: \R^{d \times d} \to \R^{d \times d}$, the deep matrix sensing problem reduces to \emph{deep matrix factorization}:
    \begin{align}
        \label{eq:deep_mf}
        \hbTheta \in \underset{\bTheta}{\arg\min}~ \ell_{\Id}(\bTheta; \bM^*)
         \coloneqq  \frac{1}{2} \|\bW_{L:1}- \bM^*\|^2_{\sfF}.
    \end{align}
    This simple setting is often used as a starting point for theoretical investigations. 
    For instance, it has been observed that in deep matrix factorization, DLNs often exhibit similar behaviors to their nonlinear counterparts~\cite{saxe2014exact}.
    \item 
    \textbf{Deep Matrix Sensing with Random Measurements.} 
    For any linear sensing operator $\cA(\cdot)$, the image $\by = \cA(\bM^*)$ can be represented as 
    \begin{align*}
        \by = \cA(\bM^*)  = \big[ \langle \bA_1, \bM^* \rangle, \cdots, \langle \bA_m, \bM^*\rangle \big] \in \R^m,
    \end{align*}
    where $\bA_i \in \bb R^{d \times d}$, $\forall i \in [m]$ denote sensing matrices and $\langle \bA_i , \bM^* \rangle \coloneqq \tr(\bA_i^{\top} \bM^*)$. 
    In the matrix sensing literature, it is common to consider a random sensing operator, for example, generated by drawing random sensing matrices $\bA_i$ with i.i.d. Gaussian entries. 
This choice is due to the fact that such a random sensing operator often satisfies desirable properties (with high probability) like the Restricted Isometry Property (RIP), which is a sufficient condition to accurately sense a low-rank $\mathbf{M}^*$ from $m < d^2$ measurements~\cite{rip}.
    \item
    \textbf{Deep Matrix Completion.} 
    Let $\Omega \subset [d] \times [d]$ and $\mathbf{y}$ be the set of elements from $\mathbf{M}^*$ with indices in $\Omega$. In this scenario, the problem reduces to the \emph{matrix completion} problem. This problem has been extensively studied in the recent decades, ignited by the Netflix prize challenge, and fuelling advances in convex relaxation heuristics~\cite{mc1} and nonconvex optimization methods~\cite{spectral}. 
    
    Formally, the sensing operator $\mathcal{A}(\cdot)$ is defined by $m = |\Omega|$ sensing matrices $\bA_{ij} = \bE_{ij}$ with $(i,j) \in \Omega$, where $\bE_{ij}$ is the $(i,j)$-th canonical basis. 
    In the literature, this problem is commonly described using the projection operator $\cP_{\Omega}: \R^{d\times d} \to \R^{d\times d}$ that only retains the entries in the index set $\Omega$, i.e., $\cP_{\Omega}(\bM^*)_{ij} = \bM^*_{ij}$ if and only if $(i,j) \in \Omega$ and $\cP_{\Omega}(\bM^*)_{ij}=0$, otherwise. 
    In our context, the deep matrix completion amounts to solving\footnote{$\cP_{\Omega} = \cA^{\dagger} \cA$ where $\cA^{\dagger}$ is the adjoint operator of $\cA$, and $\| \cP_{\Omega} (\bW_{L:1}- \bM^*) \|^2_{\mathsf{F}} = \| \cA (\bW_{L:1}- \bM^*) \|^2_{2}$.}
    \begin{align}
    \label{eq:deep_mc}
        \hbTheta \in \underset{\bTheta}{\arg\min}~ \ell_{\Omega}(\bTheta; \bM^*) \coloneqq  \frac{1}{2} \big\| \cP_{\Omega} (\bW_{L:1}- \bM^*) \big\|^2_{\sfF}.
    \end{align}    
\end{itemize}

\subsection{Efficient Low-Rank Network Compression Methods}\label{sec:method}

While overparameterization offers benefits such as decreased sample complexity and improved generalization, they come at the cost of a large increase in computational complexity. 
However, as we observed in \Cref{fig:plot_svals} and~\Cref{fig:motivation}, overparameterized networks manifest effective low-dimensionality in their training dynamics across various learning tasks.
This could be potentially exploited for an efficient training strategy that alleviates the computational burden. To this end, instead of overparameterizing each layer of the DLN, we consider a compressed DLN parameterized by
\[
    \widetilde{\bTheta} \coloneqq \left( \widetilde{\bm{W}}_1, \cdots, \widetilde{\bm{W}}_L \right)
    \quad\text{where}\quad
    \widetilde{\bm{W}}_{L} \in \R^{d\times \hat{r}},~~
    \widetilde{\bm{W}}_l \in \R^{\hat{r} \times \hat{r}},~2\leq l \leq L-1,~~\widetilde{\bm{W}}_{1} \in \R^{\hat{r}\times d},
\]
where $\hat{r}$ is any positive integer such that $\hat{r} \geq r = \rank(\bM^*)$.
The end-to-end matrix product at GD iterate $t$ with this parameterization is given by
\begin{align}\label{eq:comp_dln}
    \widetilde{\bm{W}}_{L:1}(t) = \widetilde{\bm{W}}_L(t) \cdot \ldots \cdot \widetilde{\bm{W}}_1(t) \in \R^{d\times d}.
\end{align}
Note that we do not assume precise knowledge of the true rank $r$, but only require $\hat{r} \geq r$. 
This compression reduces the total number of parameters from $Ld^2$ in the original DLN into $2d\hat{r} + (L-2)\cdot\hat{r}^2$, which is a substantial reduction when $\hat{r}$ is much smaller than $d$. 
This compressed DLN is largely motivated by Figure~\ref{fig:motivation}, where the training of the wide DLN largely occurs within an $r$-dimensional subspace. 
If we can obtain an upper bound $\hat{r} \geq r$, then we can 
essentially ``prune'' out the portion of the DLN that corresponds to the singular values that remain unchanged and small. However, doing so directly would result in large accumulative approximation error. Nonetheless, we show that for a particular chosen initialization of $\widetilde{\Theta}$, we can mitigate the error and remarkably outperform the original wide DLN.

For initialization, we choose a small constant $\epsilon > 0$ and let $\widetilde{\bm{W}}_l(0) = \epsilon \cdot \bm{I}_{\hat{r}}$ for $2 \leq l \leq L-1$; 
additionally, we initialize the left and right most factors $\widetilde{\bm{W}}_L(0)$ and $\widetilde{\bm{W}}_{1}(0)$ by extracting the top-$\hat{r}$ singular vectors of the surrogate matrix
\begin{align}
\label{eq:surrogate_ms}
    \bMsurr \coloneqq \mathcal{A}^\dagger \mathcal{A}(\bM^*) = \frac{1}{m}\sum_{i=1}^m y_i \bm{A}_i ,
    \qquad\text{where}\quad
    y_i = \langle \bm{A}_i, \bM^* \rangle,
\end{align}
scaled by $\epsilon$.\footnote{In other words, the $i$-th columns of $\widetilde{\bm{W}}_L(0)$ and $\widetilde{\bm{W}}_{1}(0)$ are the $i$-th left and right singular vectors of $\bMsurr$, respectively, up to scaling factor of $\epsilon$.} 
For a random sensing operator $\cA(\cdot)$, the singular subspaces of $\bMsurr$ are known to closely approximate those of $\bM^* \in \R^{d\times d}$ with high probability~\cite{spectral}. 
Thus, we expect $\widetilde{\bm{W}}_L(0)$ and $\widetilde{\bm{W}}_{1}(0)$ to be roughly close to the singular subspaces of $\bM^*$. Note that for the matrix completion case, $\cP_{\Omega} = \cA^{\dagger} \cA$, and so the surrogate matrix becomes $\bMsurr = \frac{1}{|\bm{\Omega}|} \cP_{\Omega}(\bM^*)$.

To update each weight matrix, we use a learning rate $\eta > 0$ to update $\widetilde{\bm{W}}_l(t)$ for $2 \leq l \leq L-1$, and a rate $\alpha \cdot \eta$ with $\alpha > 0$ to update $\widetilde{\bm{W}}_L(t)$ and $\widetilde{\bm{W}}_1(t)$. We often observe that using a scale $\alpha > 1$ to update the left and right most factors accelerates convergence.
The complete algorithm is summarized in Algorithm~\ref{alg:alg}.

To provide an intuition for the utility of spectral initialization, we conducted an experiment where we analyzed the trajectories of the factors ${\bm{W}}_L(t)$ and ${\bm{W}}_1(t)$ of the original DLN starting from orthogonal initialization. In Figure~\ref{fig:pc_dist}, we show that these two factors ultimately align with the singular vectors of the target matrix $\bM^*$. 
Interestingly, this observation is similar to that of St\"{o}ger et al.~\cite{spectral_akin}, where they note that the initial iterations of GD for two-layer matrix factorization resemble the power method. Here, we find a similar result, but it applies to the leftmost and rightmost factors of the DLN. Therefore, if we were to initialize these two factors close to the underlying singular vectors, we could achieve accelerated convergence even when compressing the wide DLN.

\begin{algorithm}[t]
\caption{Compressed DLNs (C-DLNs) for Learning Low-Dimensional Models}
\label{alg:alg}
\begin{algorithmic}[1]
    \Require loss function $\ell(\cdot)$; ~$\bm{y} \in \R^m$;~ $\epsilon, \eta, \alpha \in \R$;~ $\hat{r}, L, T \in \mathbb{N}$
    
    \State $\widetilde{\bm{W}}_l(0) \gets \epsilon \cdot \bm{I}_{\hat{r}}$, \quad $2 \leq l \leq L-1$
        \Comment{Initialize intermediate weight parameters}
    \State $\bm{U}_{\hat{r}}, \bm{V}_{\hat{r}} \gets \texttt{SVD}( \bMsurr)$
        \Comment{Compute $\bMsurr$ via Equation~(\ref{eq:surrogate_ms}) and take top-$\hat{r}$ SVD}
    \State 
    $\widetilde{\bm{W}}_L$ = $\epsilon\cdot \bm{U}_{\hat{r}}$, \,  $\widetilde{\bm{W}}_1$ = $\epsilon\cdot \bm{V}_{\hat{r}}$
        \Comment{Initialize outer weight parameters}
    \For{$t=1, \ldots, T$}
        \Comment{GD update for $T$ iterations}
    \begin{align*}
        \widetilde{\bm{W}}_{L}(t+1)       &\gets \widetilde{\bm{W}}_{L}(t) - \alpha \eta \cdot \nabla_{\widetilde{\bm{W}}_{L}} \ell(\widetilde{\bm{\Theta}}(t)\big) , \\
        \widetilde{\bm{W}}_l(t+1)   &\gets \widetilde{\bm{W}}_l(t) - \eta \cdot \nabla_{\widetilde{\bm{W}}_l}\ell \big(\widetilde{\bm{\Theta}}(t) \big),\\
        \widetilde{\bm{W}}_{1}(t+1)  &= \widetilde{\bm{W}}_{1}(t) - \alpha \eta \cdot \nabla_{\widetilde{\bm{W}}_{1}} \ell(\widetilde{\bm{\Theta}}(t)\big) ,
    \end{align*}
    \EndFor
    \State  Return $\widetilde{\bm{W}}_{L:1} = \widetilde{\bm{W}}_{L}(T)\cdot \ldots \cdot \widetilde{\bm{W}}_{1}(T)$
        \Comment{Output of compressed DLN}
\end{algorithmic}
\end{algorithm}

\begin{figure}[t!]
    \centering
     \begin{subfigure}{0.435\textwidth}
         \centering
         \includegraphics[width=\textwidth]{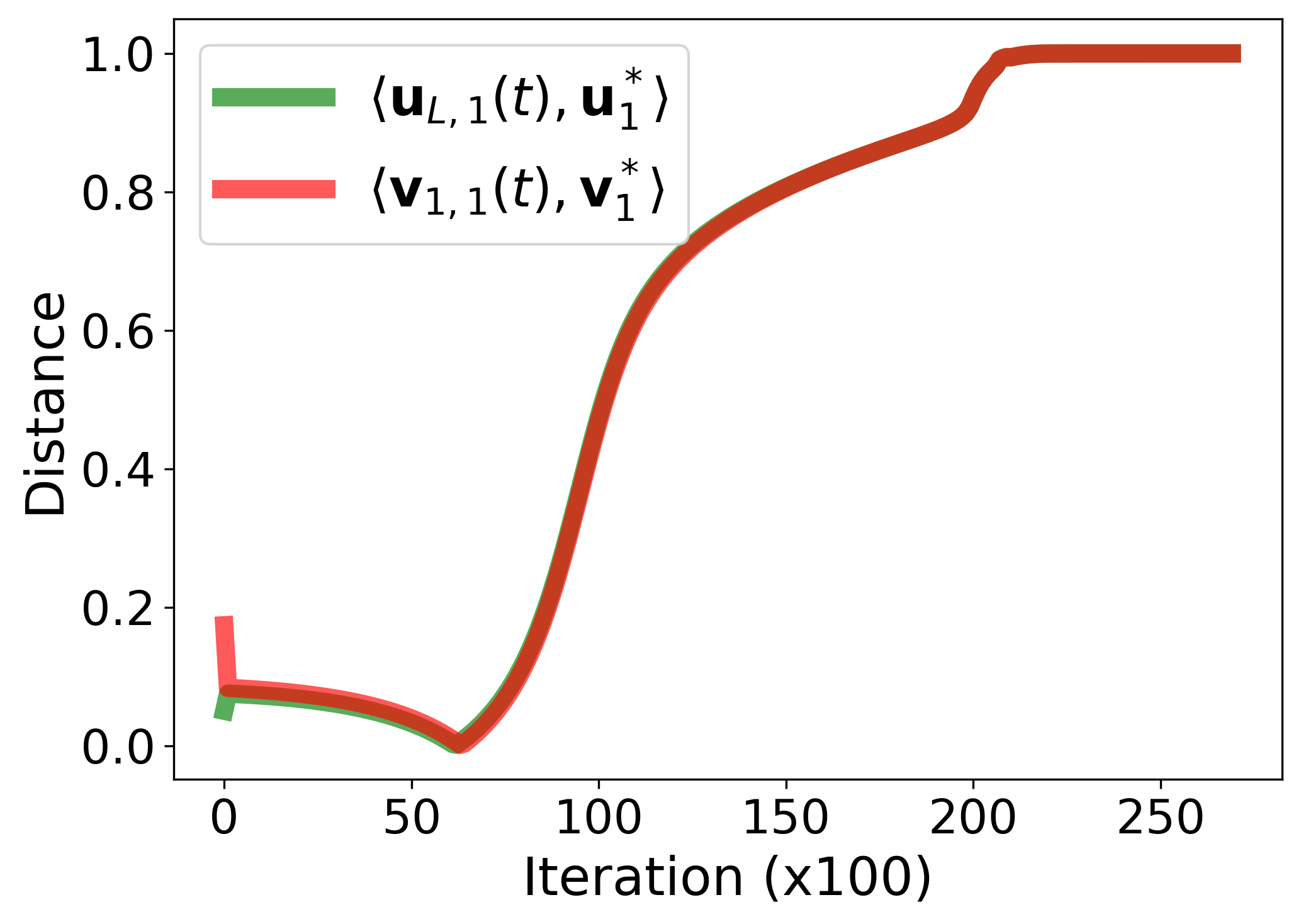}
         \caption*{Matrix Factorization}
     \end{subfigure}
     \begin{subfigure}{0.435\textwidth}
         \centering
         \includegraphics[width=\textwidth]{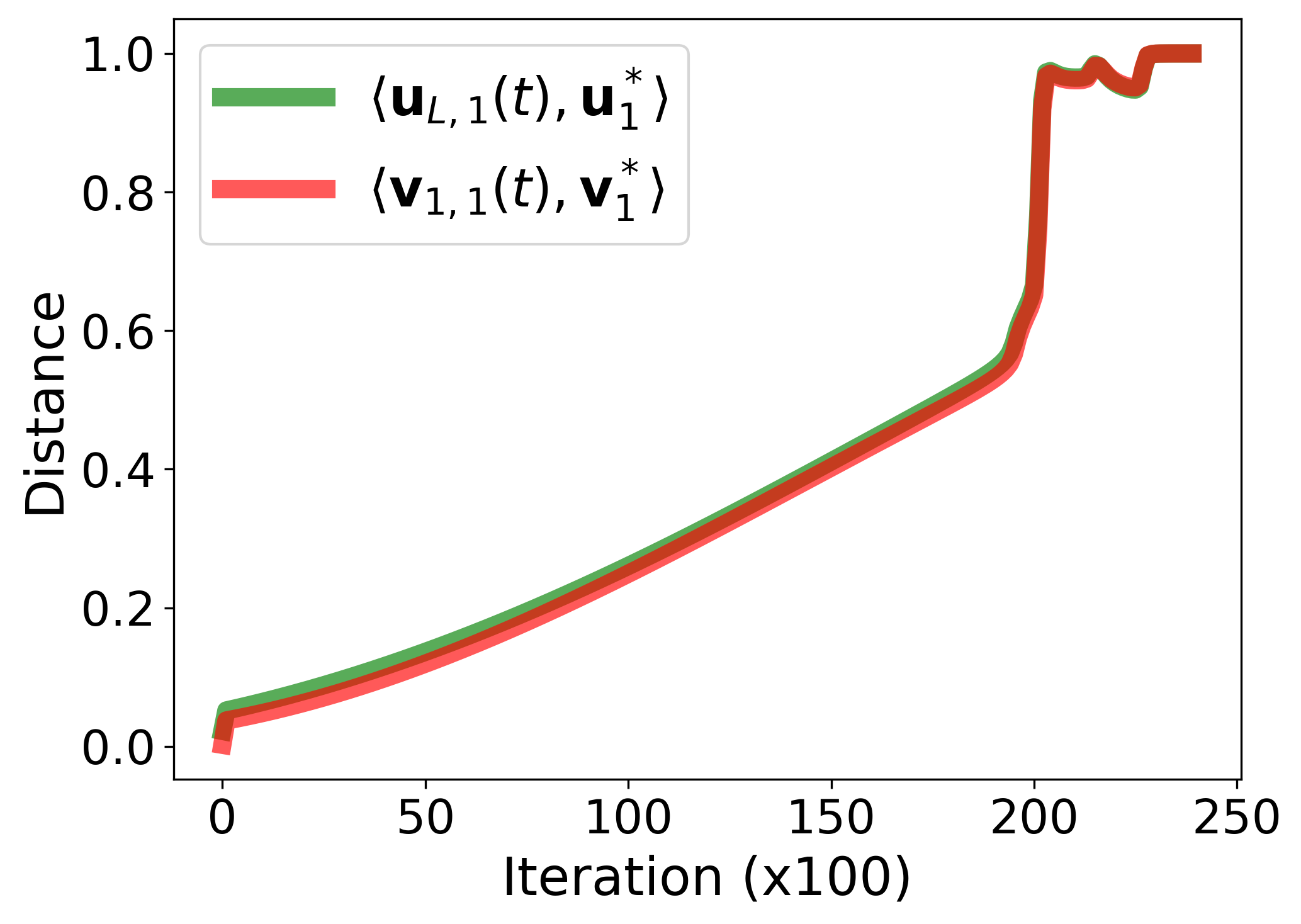}
         \caption*{Matrix Completion}
     \end{subfigure}
    \caption{\textbf{Motivating the use of spectral initialization for DLNs.} These plots measure the similarity between the first principal component of $\bm{W}_L(t)$ and $\bm{U}^*$ (and respectively $\bm{W}_1(t)$ and $\bm{V}^*$). This result shows that the left and right most factors of the DLN fit the left and right singular vectors of the target matrix $\bM^*$, respectively. }
    \label{fig:pc_dist}
\end{figure}

\subsection{Extension to Compression in Linear Layers of Nonlinear Networks}\label{sec:nonlinear_extension}

In this section, we explore leveraging our compressed network to reduce the training complexity and enhance the generalization capabilities of deep nonlinear networks, based on our observations in Figure~\ref{fig:plot_svals}. Since our compressed network takes into account the low-rank structure of the weights, it is possible that compression can lead to more efficient solutions.
Traditionally, the success of deep neural networks is attributed to their overparameterized nature, which is often achieved by increasing the width of weight matrices keeping the depth fixed. 
Interestingly, recent work by Huh et al.~\cite{huh2023lowrank} demonstrated that increasing the \emph{depth} of each weight matrix by \emph{adding linear layers} also improves generalization in terms of test accuracy, across a wide range of datasets for solving classification problems.
More concretely, consider a simple deep nonlinear network parameterized by weight matrices $\bm{\Theta} = (\bm{W}_l)_{l=1}^L$, which represents a map\footnote{Here, we omit the bias terms of the network for simplicity.} 
\[
    \psi_{\bm{\Theta}}: \bm{x} \mapsto \bm{W}_{L}\cdot\rho \left(\bm{W}_{L-1} \ldots \rho(\bm{W}_1\bm{x}) \right)
\]
where $\rho(\cdot)$ is a predetermined nonlinear activation function such as ReLU. 
Huh et al.~\cite{huh2023lowrank} showed that adding linear layers results in networks with improved generalization, namely, 
\[
    \psi'_{\bm{\Theta}}: \bm{x} \mapsto \bm{W}_{L+2}\cdot\rho \left(\bm{W}_{L+1}\bm{W}_{L}\bm{W}_{L-1} \ldots \rho(\bm{W}_1\bm{x}) \right).
\]
Observe that this amounts to overparameterizing the penultimate layer of the deep nonlinear network using a 3-layer DLN $\bm{W}_{L+1:L-1} = \bm{W}_{L+1}\bm{W}_{L}\bm{W}_{L-1}$.

Motivated by this, we explore enhancing the performance of deep nonlinear networks across a variety of architectures (e.g., MLP and vision transformers (ViT)) by (1) initially overparameterizing their penultimate layer with a DLN, followed by (2) compressing the DLN using our proposed technique. This modification is anticipated to enjoy the advantages of overparameterization through DLNs, while mitigating the increase in computational demands, thanks to our compression approach. 
However, we note that determining initial subspaces for $\widetilde{\bm{W}}_{L}(0)$ and $\widetilde{\bm{W}}_{l}(0)$ in this setting is not as straightforward as in the low-rank matrix recovery context. 
As a first step towards applying compression techniques to deep nonlinear networks, we propose initializing the subspaces by using the singular subspaces of the cross-correlation matrix, $\bm{M}^{\text{corr}} \coloneqq \bm{YX}^{\top} \in \R^{d_y \times d_x}$ where applicable and otherwise using random subspaces. When the number of classes $d_y$ is smaller than the feature dimension $d_x$, this matrix $\bm{M}^{\text{corr}}$ is a low-rank matrix with $\rank \bm{M}^{\text{corr}} \leq d_y \ll d_x$, permitting the choice $\hat{r} = d_y$. For the setting of random subspaces, we generally observe that a choice of $\hat{r} \geq d_y$ suffices. 

%% file: contents/compression.tex
\section{Theoretical Investigations}
\label{sec:theory}

In this section, we provide theoretical insights into the superior performance of our compressed DLN compared to the original overparameterized network based upon the simplified deep matrix factorization setting. 
We reveal the benefits of the spectral initialization and the incremental learning phenomenon in achieving accelerated convergence (in iteration complexity) for our compressed network.

\subsection{The Benefits of Spectral Initialization}
Let $\bM^* \in \R^{d\times d}$ be a matrix of rank $r$ and $\bM^* = \bm{U}^* \bm{\Sigma}^* \bm{V}^{*\top}$ be a singular value decomposition (SVD) of $\bM^*$. 
We consider the deep matrix factorization setting, where $\cA = \Id$, i.e., we have full observation of $\bM^*$. Here, $\bMsurr = \bM^*$ and so 
Algorithm \ref{alg:alg} initializes $\widetilde{\bm{W}}_{L}(0) = \epsilon \cdot \bm{U}^*_{\hat{r}}$ and $\widetilde{\bm{W}}_{1}(0) = \epsilon \cdot \bm{V}^{*\top}_{\hat{r}}$, the top-$\hat{r}$ singular subspaces of $\bM^*$ themselves. This leads to the compressed deep matrix factorization problem
\begin{align}
\label{eq:comp_deep_mf}
    \underset{\widetilde{\bm{\Theta}}}{\min} \, \ell(\widetilde{\bm{\Theta}}(t)) = \frac{1}{2} \|\widetilde{\bm{W}}_{L:1}(t) - \bM^*\|^2_{\mathsf{F}},
\end{align}
where the intermediate layers are initialized to $\widetilde{\bm{W}}_l(0) = \epsilon \cdot \bm{I}_{\hat{r}}$ for $2 \leq l \leq L-1$.
In Figure~\ref{fig:pc_dist}, we showed that the left and rightmost factors of the original DLN align with the left and right singular vectors of the target matrix $\bM^*$ more closely throughout GD. 
This observation implies that our particular choice of initialization could accelerate the training process. 
We substantiate this observation by proving that this initialization has two advantages: (1) the compressed DLN has a low-dimensional structure as outlined in Theorem~\ref{thm:parsimony} and (2) the compressed DLN incurs a lower recovery error at initialization than the original DLN, as shown in Corollary~\ref{cor:init}.

\begin{theorem}
\label{thm:parsimony} 
Let $\bM^* \in \R^{d\times d}$ be a matrix of rank $r$ and $\bM^* = \bm{U}^* \bm{\Sigma}^* \bm{V}^{*\top}$ be a SVD of $\bM^*$. Suppose we run Algorithm~\ref{alg:alg} to update all weights $\left( \widetilde{\bm{W}}_l \right)_{l=1}^L$ of Equation~(\ref{eq:comp_deep_mf}), where $\mathcal{A} = \text{Id}$. Then, the end-to-end compressed DLN possesses low-dimensional structures, in the sense that for all $t\geq 1$,  $\widetilde{\bm{W}}_{L:1}(t)$ admits the following decomposition:
\begin{align}
    \widetilde{\bm{W}}_{L:1}(t) = \bm{U}^*_{\hat{r}} 
    \begin{bmatrix}
    \bm{\Lambda}(t)  &\bm{0} \\
        \bm{0} & \beta(t)^L\cdot\bm{I}_{\hat{r} - r}
    \end{bmatrix}
    \bm{V}_{\hat{r}}^{*\top},
\end{align}
where $\bm{\Lambda}(t) \in \R^{r\times r}$ is a diagonal matrix with entries $\lambda_i(t)^L$, where
\begin{align}
    \lambda_i(t) = \lambda_i(t-1) \cdot \left(1 - \eta \cdot (\lambda_i(t-1)^L - \sigma_i^*) \cdot \lambda_i(t-1)^{L-2} \right), \quad 1 \leq i \leq r,
\end{align}
with $\lambda_i(0) = \epsilon$ and $\sigma^*_i$ is the $i$-th diagonal entry of $\bm{\Sigma}^*$ and
\begin{align}
    \beta(t) = \beta(t-1) \cdot \left( 1 - \eta \cdot \beta(t-1)^{2(L-1)} \right), 
\end{align}
with $\beta(0) = \epsilon$.
\end{theorem}

\paragraph{Remarks.} We defer all proofs to Appendix~\ref{sec:proofs_app}. 
By using our initialization technique, Theorem~\ref{thm:parsimony} allows us to characterize the GD updates of the compressed DLN, which also constitutes a valid SVD. This also shows that the compressed DLN directly finds low-rank solutions of rank $r$, as the last $\hat{r} - r$ singular values is a decreasing function where $\beta(t) \leq \epsilon^L$ for a small constant $\epsilon > 0$.

We would also like to note a couple of points regarding the relationship between Theorem~\ref{thm:parsimony} and existing theory for the original DLN \cite{yaras2023law}. 
The main difference of our result from that of Yaras et al.~\cite{yaras2023law} is that, through the use of spectral initialization, we are able to characterize the dynamics of \emph{all} of the singular values, whereas Yaras et al. \cite{yaras2023law} only characterizes the behavior of the trailing $d-2r$ singular values. The result by Arora et al.~\cite{arora2019implicit} has a similar flavor — if we consider infinitesimal steps of their gradient flow result regarding the behavior of the singular values, we can recover the discrete steps as outlined in \Cref{thm:parsimony}.

With \Cref{thm:parsimony} in place, an immediate consequence is that the compressed DLN exhibits a lower recovery error than the original DLN at the initialization $t=0$, as outlined in Corollary~\ref{cor:init}.
\begin{corollary}
\label{cor:init}
Let $\bm{W}_{L:1}(0)$ denote the original DLN at $t=0$ initialized with orthogonal weights according to Equation~(\ref{eq:orth_init}) with $\mathcal{A} = \text{Id}$ and $\widetilde{\bm{W}}_{L:1}(0)$ denote the compressed DLN at $t=0$. Then, we have
\begin{align}
\label{eq:recovery_error}
        \| \bm{W}_{L:1}(0) - \bM^*\|^2_{\mathsf{F}} \geq  \| \widetilde{\bm{W}}_{L:1}(0) - \bM^*\|^2_{\mathsf{F}}.
\end{align}
\end{corollary}
This result can easily be established by applying Theorem~\ref{thm:parsimony} at GD iterate $t=0$. In the next section, we will demonstrate that the inequality in Equation~(\ref{eq:recovery_error}) holds for all GD iterations $t\geq 0$, which involves leveraging the incremental learning phenomenon along with an analysis using gradient flow.

\subsection{The Benefits of Incremental Learning}

\label{sec:benefits_inc}

In this section, we aim to establish that the inequality in Equation~(\ref{eq:recovery_error}) holds for all GD iterations $t\geq 0$ for the deep matrix factorization case. To prove such a result, we assume that both the original DLN and the compressed DLN undergo incremental learning, in the sense that the singular values of both networks and their respective singular vectors are fitted sequentially. This assumption is stated with more rigor in Assumption~\ref{ass:incremental}.

\begin{assumption}
\label{ass:incremental}
Let $\bm{W}_{L:1}(t) \in \R^{d\times d}$ denote the end-to-end weight matrix at GD iterate $t$ with respect to Equation~(\ref{eq:dln_setup}) with $\mathcal{A} = \text{Id}$ and $\bM^* \in \R^{d\times d}$ be the target matrix with rank $r$. Then, GD follows an incremental learning procedure in the sense that there exist small constants $c_{\text{val}}, c_{\text{vec}} \in [0, 1]$ and time points $t_1 \leq t_2 \leq \ldots \leq t_r \in \mathbb{R}$ such that
\begin{align}
    \left(\sigma_i(\bm{W}_{L:1}(t)) - \sigma_i(\bm{M}^*)\right)^2 &\leq c_{\text{val}}, 
    \\
    \langle \bm{u}_i(t), \bm{u}_i^* \rangle &\geq 1-c_{\text{vec}}, 
    \\
    \langle \bm{v}_i(t), \bm{v}_i^* \rangle &\geq 1-c_{\text{vec}}, 
\end{align}
for all $t > t_i$ and $\lim_{\epsilon \to 0} \, \sigma_i(\bm{W}_{L:1}(t)) = 0$ for all $t < t_i$, where $\epsilon > 0$ is the initialization scale. 
\end{assumption}
\paragraph{Remarks.} %
Assumption~\ref{ass:incremental} states that there exists a sequence of time points where the $i$-th principal components of both networks are fitted to a certain precision,  
while the singular values associated with the remaining principal components remain close to zero, assuming that the initialization scale $\epsilon$ is small. This assumption has been widely studied and adopted for the two-layer case~\cite{incremental2, jiang2023algorithmic}, with some results extending to the deep matrix factorization case~\cite{li2020towards,jacot2021saddle, chou2023gradient, gidel}. 
In Figure~\ref{fig:incre_mf_r5}, we further show the occurrence of this phenomenon and demonstrate its validity with extensive experimental results in Appendix~\ref{sec:incremental_app}. In Appendix~\ref{sec:incremental_app}, we also show that generally, we have $c_{\text{vec}} = c_{\text{val}} = 0$, which will play a role in our proofs as well. 
While we initially assume this phenomenon for deep matrix factorization, our extensive experiments also confirm that Assumption~\ref{ass:incremental} holds for low-rank matrix sensing and completion.
By using this phenomenon, we can analyze the singular values of the original and compressed DLN one at a time, and show that the singular values of the compressed network are fitted more quickly than those of the original network, leading to faster convergence in terms of iteration complexity. Our main result is established in Theorem~\ref{thm:recovery_mf}.
In Section~\ref{sec:incremental_app}, we also extensively demonstrate that Assumption~\ref{ass:incremental} holds for low-rank matrix sensing and completion.

\begin{figure}[t!]
    \centering
    \includegraphics[width=0.95\textwidth]{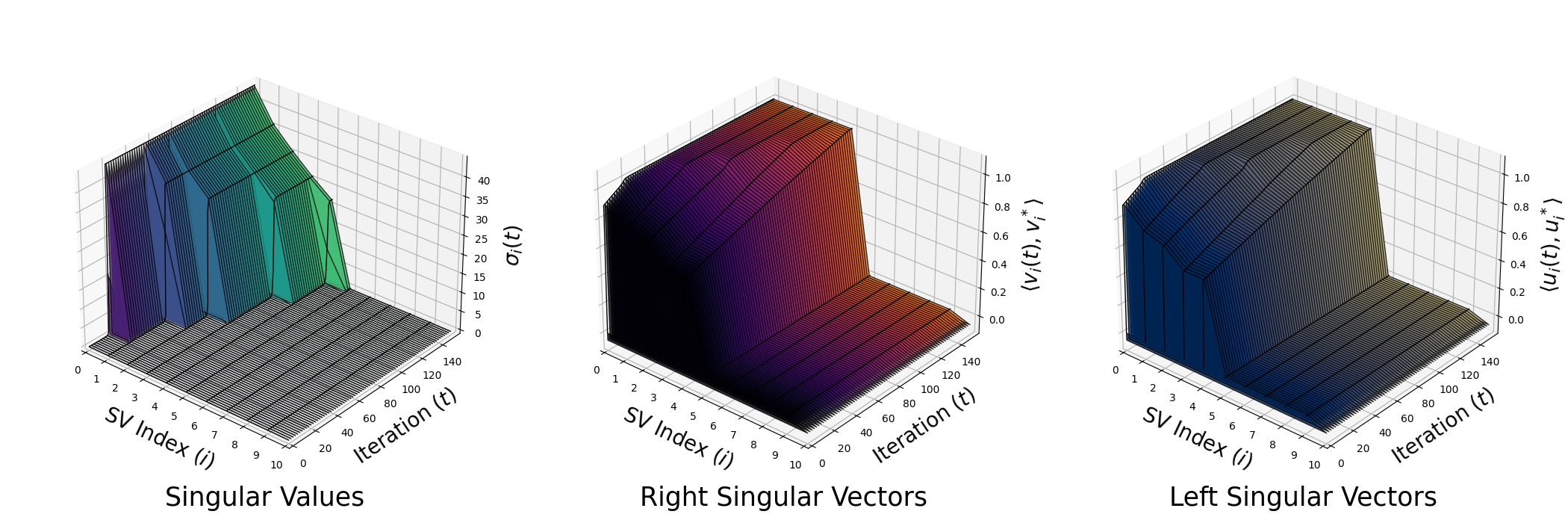}
    \caption{\textbf{Occurrence of the incremental learning phenomenon in deep matrix factorization.} We observe that the first $r=5$ singular values are fitted incrementally, along with their respective singular subspaces, corroborating Assumption~\ref{ass:incremental}.}
    \label{fig:incre_mf_r5}
\end{figure}

\begin{theorem}
\label{thm:recovery_mf}
      Let $\bM^* \in \R^{d\times d}$ be a rank-$r$ matrix and let $\hat{r} \in \mathbb{N}$ such that $\hat{r} \geq r$. Suppose that we run gradient flow with respect to the original DLN in Equation~(\ref{eq:dln_setup}) and with respect to the compressed network defined in Equation~(\ref{eq:comp_deep_mf}) with $\mathcal{A} = \text{Id}$. Then, if Assumption~\ref{ass:incremental} holds such that $c_{\text{vec}} = 0$, we have that $\forall t\geq 0$,
    \begin{align}
        \|\bm{W}_{L:1}(t) - \bM^*\|^2_{\mathsf{F}} \geq \|\widetilde{\bm{W}}_{L:1}(t) - \bM^*\|^2_{\mathsf{F}}.
    \end{align}
\end{theorem}
\paragraph{Remarks.} The proof relies on analyzing the evolution of singular values of $\bm{W}_{L:1}(t)$ and $\widetilde{\bm{W}}_{L:1}(t)$ using gradient flow. Through this analysis, we demonstrate that the singular values of the compressed network are fitted more quickly than those of the original DLN throughout all iterations of GD, thereby highlighting the benefits of spectral initialization and our compression technique.
We remark that there is a slight discrepancy between our algorithm and the analysis. Our algorithm employs discrete gradient steps to update the weight matrices. However, it is well-established that using differential equations (and hence gradient flow) for theoretical analysis has a rich history, and it is known that discrete gradient steps approximate the gradient flow trajectories as long as the step size is sufficiently small~\cite{history, arora2018optimization}. 

To extend our analysis to the deep matrix sensing case, we would need to prove a variant of Theorem~\ref{thm:parsimony}, where we characterize the dynamics of the singular values of the compressed network for all GD iterations. However, due to the sensing operator, we only obtain an estimate of the singular subspaces of the target matrix at initialization, and so we do not expect the compressed network to remain diagonal, making the characterization of the dynamics much more challenging. We leave it for future work to extend it to this case, and believe that some promising avenues might be to impose some structure on $\mathcal{A}(\cdot)$, such as RIP~\cite{rip}.

%% file: contents/experiments.tex
\section{Experiments}
\label{sec:experiments}

In this section, we showcase our experimental results for the problems discussed in Section~\ref{sec:problem}. In Section~\ref{sec:lowrank_exp}, we present results for solving low-rank matrix recovery problems, including matrix sensing and completion, on both synthetic and real data. For matrix factorization, our goal is to demonstrate practical validity of Theorem~\ref{thm:recovery_mf}.
In Section~\ref{sec:nonlinear_exp}, we show that adding linear layers to a deep network indeed improves generalization, and provide results on compressing the linear layers using our proposed method.

\subsection{Matrix Recovery Problems}
\label{sec:lowrank_exp}
Throughout all of the experiments in this section, we use a DLN of depth $L=3$ and small initialization scale $\epsilon=10^{-3}$. To quantitatively measure the performance between the original and compressed DLN, we use the recovery error defined as
\begin{align*}
    \text{Recovery Error} = \frac{\|\widehat{\bm{W}} - \bM^*\|_{\mathsf{F}}}{\|\bM^*\|_{\mathsf{F}}},
\end{align*}
where $\widehat{\bm{W}}$ is an estimate of the target matrix. 

\paragraph{Deep Matrix Factorization.} For deep matrix factorization, we synthetically generate a data matrix $\bM^* \in \R^{d\times d}$ with $d=100$ and rank $r=10$. We use the left and right singular vectors of $\bM^*$ for $\widetilde{\bm{W}}_{L}(0) \in \mathcal{O}^{d \times \hat{r}}$ and $\widetilde{\bm{W}}_{1}(0) \in \mathcal{O}^{\hat{r} \times d}$, where we choose $\hat{r} = 20$ as our upper bound on the rank $r$. We run GD with $\eta = 10$ for the learning rate and $\alpha = 5$ for the scale. In Figure~\ref{fig:mf_r5}, we can observe that the compressed network maintains a lower recovery error throughout all iterations of GD, corroborating Theorem~\ref{thm:recovery_mf}. As a result, this leads to two advantages: (1) a reduction in the number of iterations to converge to a specific threshold when using the compressed DLN, and (2) a further reduction in time complexity, as each iteration of the compressed DLN is much faster than that of the original DLN.


\begin{figure}[t!]
    \centering
    \includegraphics[width=0.75\textwidth]{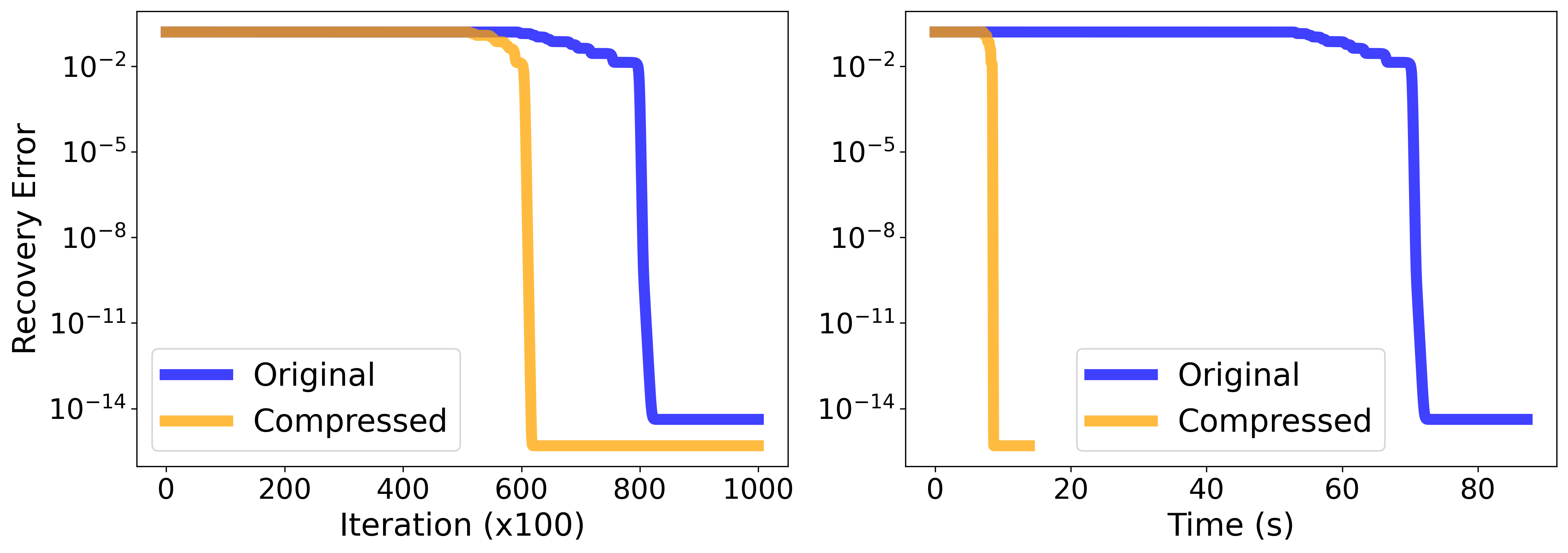}
    \caption{\textbf{Empirical results on deep linear matrix factorization.} Left: Shows that our compressed network achieves a lower recovery error than the original network, corroborating \Cref{thm:recovery_mf}. Right: Demonstrates the speed up over the original network.}
    \label{fig:mf_r5}
\end{figure}

\paragraph{Deep Matrix Completion.} Here, we present our results for deep matrix completion and present results for deep matrix sensing in Appendix~\ref{sec:deep_ms_app}.
Our goal is to show that choosing the singular subspaces of the surrogate matrix in Equation~(\ref{eq:surrogate_ms}) serve as good initial points for $\widetilde{\bm{W}}_L(0)$ and $\widetilde{\bm{W}}_1(0)$.

Given $\widetilde{\bm{W}}_L(0)$ and $\widetilde{\bm{W}}_1(0)$, we also compare the performance of our compressed network to AltMin~\cite{jain2012lowrank}, which involves alternatingly minimizing over just the factor matrices to construct $\widehat{\bm{W}} = \widetilde{\bm{W}}_L\widetilde{\bm{W}}_1$.
Firstly, we compare these algorithms using synthetic data, where we follow the setup in deep matrix factorization. For the observation set $\bm{\Omega}$, we consider the ``missing completely at random'' (MCAR) setting, where each entry of $\bm{\Omega}$ is Bernoulli with probability $p$.
We choose $p=0.3$ so that roughly $30\%$ of the observations are observed. We run GD with learning rate $\eta = 10$ and $\alpha=5$. In Figure~\ref{fig:mc_synthetic}, we observe the same trends as seen in the deep matrix factorization case, where our compressed DLN consistently exhibits lower recovery error than the original DLN. Additionally, it is evident that AltMin fails to recover the underlying matrix completely, as the rank is overspecified. We observe that while the training loss goes to zero, the recovery error does not decrease, as there are insufficient measurements for recovery using this parameterization. To efficiently use AltMin (or both networks for the $L=2$ case), one would need to obtain a more accurate estimate of the rank $\hat{r}$ for recovery.

Next, we compare these algorithms on the MovieLens 100K dataset~\cite{movielens}. Since the MovieLens dataset is not precisely a low-rank matrix, this experiment also serves to demonstrate the performance of DLNs (as well as compressed DLNs) on approximately low-rank matrix completion.
We randomly choose $80\%$ of the samples to train the network and test on the remaining $20\%$. For the hyperparameters, we choose $\hat{r} = 10$, $\eta = 0.5$, and $\alpha = 5$. As depicted in Figure~\ref{fig:mc_movielens}, we find that our compressed network achieves the same recovery error as the original DLN in less than $5\times$ the time. While AltMin initially finds solutions with lower recovery error at a faster rate, its recovery error eventually plateaus, whereas both DLNs can find solutions with overall lower recovery error.



\begin{figure}[t!]
    \centering
    \includegraphics[width=0.7\textwidth]{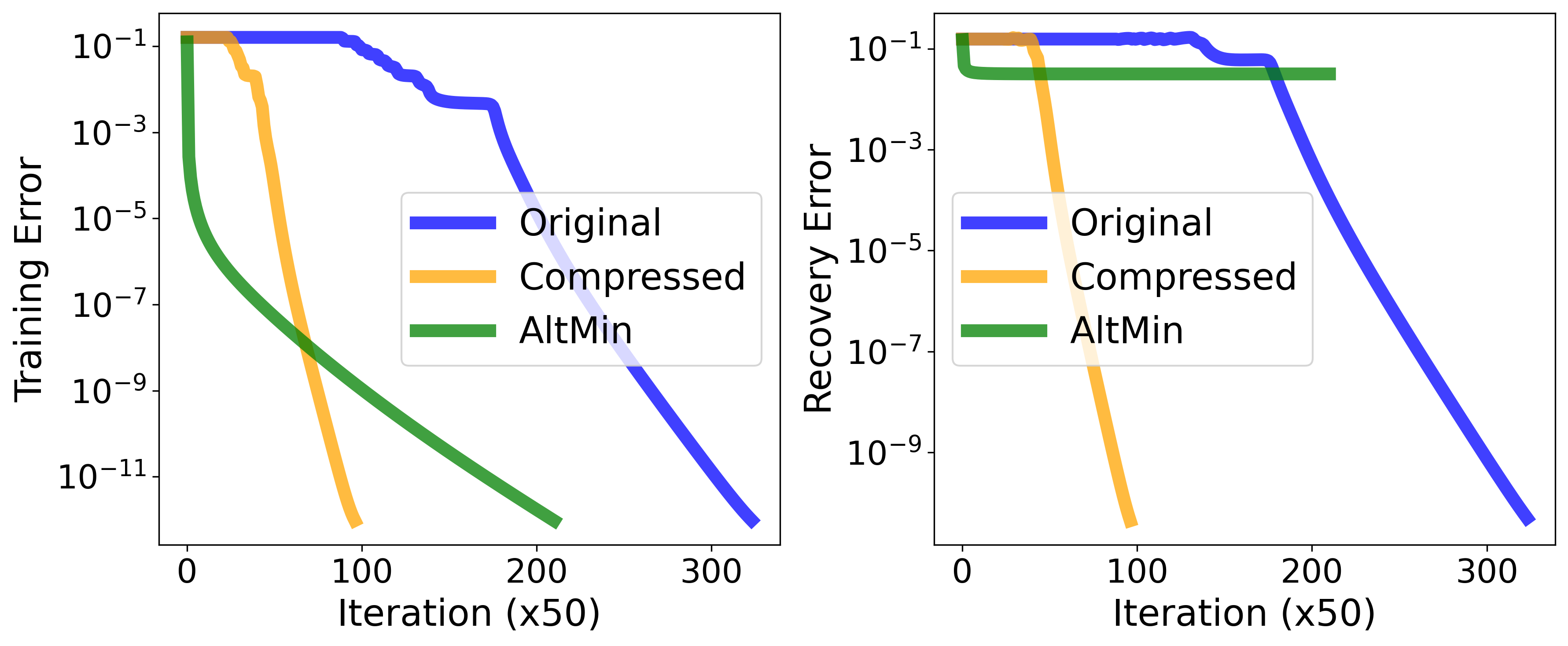}
    \caption{\textbf{Results on matrix completion with synthetic data.} The left figure demonstrates the training error, whereas the right figure shows the recovery error for completing a rank $r=10$ matrix with only $30\%$ observed entries. This experiment demonstrates the superiority of our method over the original DLN and AltMin~\cite{jain2012lowrank}.}
    \label{fig:mc_synthetic}
\end{figure}

\begin{figure}[t!]
    \centering
    \includegraphics[width=0.7\textwidth]{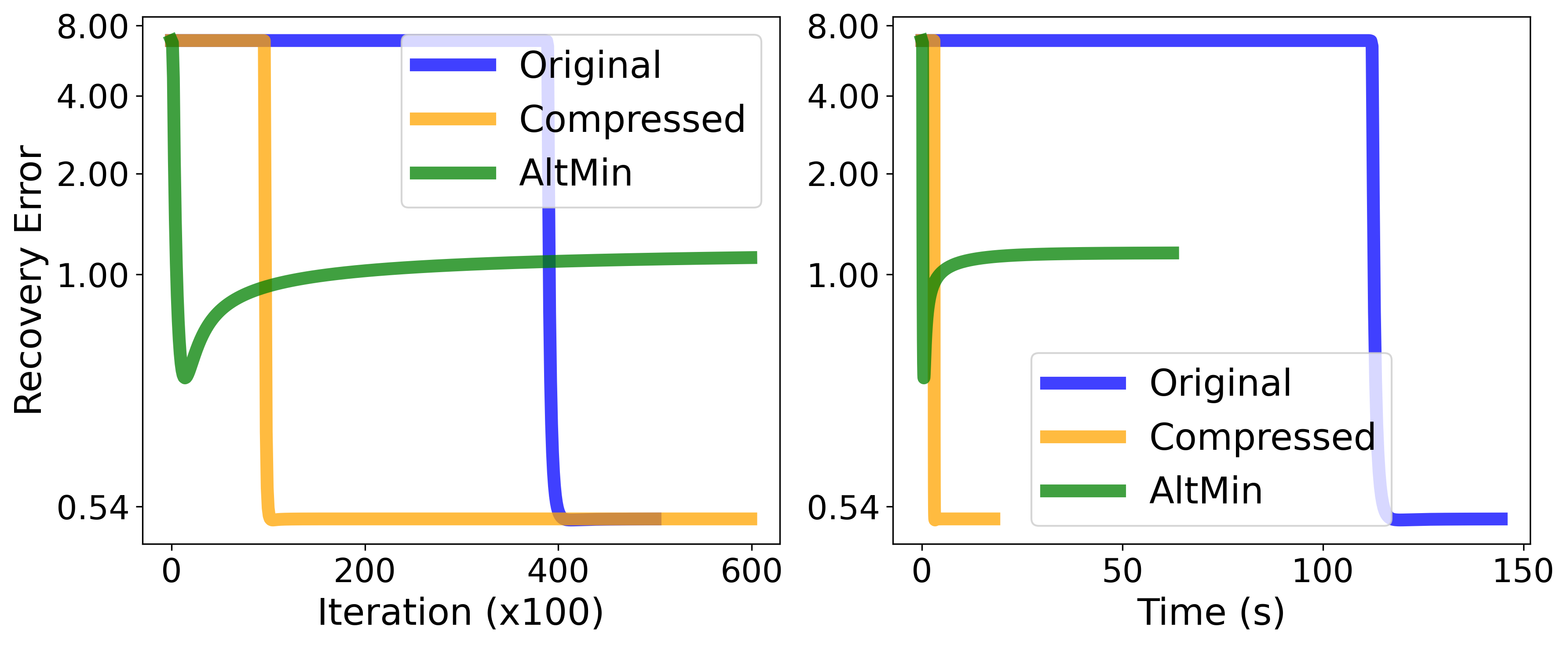}
    \caption{\textbf{Results on matrix completion with the MovieLens dataset~\cite{movielens}.} The left figure demonstrates the recovery error, whereas the right figure shows the training time with $\hat{r} = 10$. Similar to the synthetic case, this result highlights the effectiveness of our method.}
    \label{fig:mc_movielens}
\end{figure}

\subsection{Applications of DLNs for Deep Nonlinear Networks}
\label{sec:nonlinear_exp}

To demonstrate the application of our compressed DLN on deep nonlinear networks,
we consider the setting in Section~\ref{sec:nonlinear_extension}, where we overparameterize the penultimate layer using a DLN.
We consider this additional overparameterization using two network architectures: MLPs and ViTs.
For MLPs, we train two MLPs: (1) one MLP with 3 hidden layers and one linear penultimate layer and (2) one MLP with $3$ hidden layers and one $3$-layer DLN as the penultimate layer.
Mathematically, this amounts to the compressed network
\begin{align*}
    \psi_{\widetilde{\bm{\Theta}}}(\bm{x}) = \bm{W}_6\cdot\rho(\underbrace{\widetilde{\bm{W}}_5\widetilde{\bm{W}}_4\widetilde{\bm{W}}_3}_{\widetilde{\bm{W}}_{5:3}} \cdot\rho(\bm{W}_{2} \rho(\bm{W}_1\bm{x})).
\end{align*}
For ViT, we consider 
a smaller variant of ViT-base, which consists of $6$ alternating layers of multi-headed self-attention and MLP blocks. We set the token dimension as $512$ and the MLP dimension as $3072$ and compress last linear layer in the $6$-th MLP block, which can also be seen as the penultimate layer of the ViT.
Here, our objective is two-fold: (1) to demonstrate that overparameterizing the penultimate layer has better generalization capabilities and (2) to show that the compressed networks have similar performance to the over-parameterized networks while significantly reducing the training time.

We illustrate this by training the MLPs on the FashionMNIST dataset and training the ViTs on the CIFAR-10 dataset. For $\hat{r}$, we choose $\hat{r} = 4d_y$, where $d_y$ denotes the number of classes in the dataset. 
For the MLPs, we run GD with $\eta = 5\times10^{-3}$ and $\alpha = 1$ and used $\eta=1\times10^{-4}$, $\alpha=1$ with cosine annealing for the ViTs. Both models were tested $5$ times starting from random initialization to account for variability. In Table~\ref{table:quant_nonlinear}, we observe that our compressed network can achieve the highest accuracy for MLPs on average, with very competitive results for ViTs. However, the compressed network takes significantly less time to train while storing fewer parameters.
Overall, these results highlight the effectiveness of our approach, demonstrating that one can reduce runtime and memory without sacrificing the performance of deeper and wider models.

\begin{table}[t!]
\begin{center}

\resizebox{\textwidth}{!}{%
\begin{tabular}{c c|c c c c c } 
\hline
\multicolumn{2}{c|}{Method} & Test Accuracy ($\%$) & Time & MACs & Memory (Train) & \# of Parameters\\
\hline
&Original & $89.87\pm 0.216$ & \underline{$83.02 \pm 20.87$} sec  & $\underline{2.370 \times 10^7}$ & \underline{25.80} MB & \underline{1.850} M\\ 
MLP&DLN& $\underline{90.20\pm 0.040}$ & $95.82\pm 12.29$ sec & $3.940 \times 10^7$ & 35.90 MB & 3.080 M\\ 
&C-DLN & $\mathbf{90.28\pm 0.104}$ & $\mathbf{54.85\pm 2.950}$ sec & $\mathbf{1.670 \times 10^7}$ & \textbf{21.60} MB & \textbf{1.300} M\\ 
\hline
&Original & $84.72 \pm 0.120$ & \underline{$32.80 \pm 2.786$} min & $\underline{2.100 \times 10^{10}}$ & \underline{$7.110 $} GB & \underline{25.20} M\\ 
ViT&DLN & $\mathbf{84.90 \pm 0.230}$ & $45.80 \pm 2.638 $ min & $3.640 \times 10^{10}$ & $7.811$ GB & 44.10 M\\ 
&C-DLN & $\underline{84.89 \pm 0.187}$ & $\mathbf{31.40 \pm 1.855}$ min & $\mathbf{1.980 \times 10^{10}}$ & $\mathbf{7.090}$ GB & \textbf{23.80} M\\ 
\hline
\end{tabular}}
\end{center}
\caption{\textbf{Quantitative results for nonlinear networks.} DLN and C-DLN denote networks that are overparameterized using their respective models. The reported time is the amount of time taken for the models to achieve $99\%$ test accuracy. Note that the reduction in memory and MACs is not substantial as we only consider compressing the penultimate layer. Best results are in bold and second best results are underlined.}
\label{table:quant_nonlinear}
\end{table}

\section{Conclusion}

In this work, we proposed an efficient technique to compress overparameterized networks by studying the learning dynamics of DLNs. Our method involved reducing the width of the intermediate layers along with a spectral initialization scheme that improved convergence compared to its wider counterpart. We rigorously demonstrated that our compressed network has a smaller recovery error than the original network for the deep matrix factorization case,
and empirically verified this for deep matrix sensing. We also further demonstrated its applicability to deep nonlinear networks.
We believe that our work opens doors to many exciting questions: the theoretical analysis of the learning trajectory for deep matrix sensing, and the practical application of further improving low-rank training for deep nonlinear networks. Finally, we conclude with a survey of related works and subsequently discussing their connections to our work.

\paragraph{Deep Linear Networks.} Despite their simplicity, deep linear networks have been widely adopted for theoretical analysis, as it has been observed that they share similar behavioral characteristics as their nonlinear counterparts~\cite{saxe2014exact}. Some examples include the study of their optimization landscape~\cite{opt1, opt2, opt3, dln_benign} and the study of understanding feature representation in deep networks~\cite{nc1, nc2,jiang2023generalized, lr1,lr2,wang2022linear,zhou2022optimization,zhou2022all}. Our work highlights that these deep linear models are not only useful analytical tools but also powerful models for solving low-rank matrix recovery tasks. This observation is built upon some of the work done by Arora et al.~\cite{arora2019implicit}, where they show that deeper models tend towards more accurate solutions for these matrix recovery problems in settings where the number of observations are very limited. The most relevant works to this study are those conducted by Yaras et al.~\cite{yaras2023law} and Khodak et al.~\cite{khodak2021initialization}. We were unaware of the work by Khodak et al.~\cite{khodak2021initialization} at the time of writing this paper, where they explore the advantages of spectral initialization and weight decay for deep nonlinear networks. Their methods involve assuming a Gaussian prior on the weights of deep nonlinear networks, which they then apply spectral initialization with weight decay to train the weights of such networks. However, their methods are not directly applicable to the problems addressed in this paper, as a Gaussian prior alone is insufficient to realize the benefits of the compressed network as demonstrated in this study. This point is further emphasized by our theory -- the spectral initialization step must incorporate some function of the data for it to be meaningful. The study by Yaras et al.~\cite{yaras2023law} investigates the learning dynamics of DLNs starting from orthogonal initialization in the deep matrix factorization case. Their theory includes an additional weight decay parameter, which can also be extended to our theory.

\paragraph{Implicit Bias in Overparameterized Networks.} Recently, there has been an overwhelming amount of work devoted to understanding the generalization capabilities of deep networks~\cite{gen1, gen2, gen3, neyshabur2017implicit}. These overparameterized models have shown to produce solutions that generalize exceptionally well despite being overparameterized, seemingly contradicting traditional learning theory~\cite{allenzhu2020learning}. Their effectiveness is often attributed, at least in part, to the implicit bias (or regularization) inherent in their learning dynamics, which favors certain solutions. These solutions are often divided into two categories: ``simple’’ solutions, where GD tends to find solutions that satisfy known equations or are solutions to functions~\cite{simple1, simple2, simple3, shenouda2023vectorvalued}, and low-rank solutions, where GD is biased towards solutions that are inherently low-rank, assuming that the underlying problem is low-rank~\cite{arora2019implicit, bias1, bias2}. There also exists a line of work that leverages these implicit properties to improve robustness~\cite{robust1, robust2}. Our work builds upon the low-rank bias of GD and shows that the bias in GD for training DLNs can be leveraged to efficiently solve low-rank matrix recovery problems.

\paragraph{Low-Rank Training and Adaptation.} 

Low-rank training refers to modeling the weight updates of networks (whether shallow or deep) as a product of low-rank matrices, rather than updating the entire matrix itself. By updating the low-rank matrices, one can reduce training costs and enhance generalization by exploiting its instrinsic structure. This method has a long and rich history, dating back to the Burer-Monteiro factorization~\cite{burer} and including efficient implementations of alternating minimization~\cite{ma2023provably, spectral}. Our work can be viewed as an improvement upon these techniques, as Arora et al.~\cite{arora2019implicit} have observed that deeper models are more favorable for modeling low-rank matrices. There is also a body of work related to low-rank adaptation, which generally involves low-rank fine-tuning of large models~\cite{hu2021lora, wang2023cuttlefish, zhao2023inrank, lialin2023stack,zhai2023investigating}. For instance, Hu et al.~\cite{hu2021lora} proposed Low-Rank Adaptation (LoRA), which fine-tunes large language models (LLMs) by assuming that the weight updates have a low-dimensional structure and modeling them as a product of two matrices. This method has shown to significantly reduce memory complexity with only a minimal tradeoff in test accuracy.
There also have been attempts at extending LoRA by training low-rank matrices from scratch, as such techniques not only reduce the training costs but also leads to possible performance gain by restricting the training dynamics to a low-dimensional manifold~\cite{wang2023cuttlefish, zhao2023inrank, lialin2023stack,zhai2023investigating}.
We believe that our work can enhance the utility of algorithms of these algorithms,
where we can model the weight updates using a DLN. By doing so, one can reduce the training time using our compression technique, while enjoying the benefits of deeper models.

\section*{Acknowledgement}

SMK, DS, and QQ acknowledge support from NSF CAREER CCF-2143904, NSF CCF-2212066, NSF CCF-2212326, and NSF IIS 2312842, ONR N00014-22-1-2529, an AWS AI Award, a gift grant from KLA, and MICDE Catalyst Grant. SMK and LB acknowledge support from DoE award DE-SC0022186, ARO YIP W911NF1910027, and NSF CAREER CCF-1845076. Results presented in this paper were obtained using CloudBank, which is supported by the NSF under Award \#1925001. SMK would also like to thank Can Yaras for fruitful discussions. ZZ would like to thank Zaicun Li for insights and justifications regarding mathematical proofs.

%% file: contents/additional_results.tex
\onecolumn
\par\noindent\rule{\textwidth}{1pt}
\begin{center}
{\Large \bf Appendices}
\end{center}
\vspace{-0.1in}
\par\noindent\rule{\textwidth}{1pt}

\section{Additional Results}
In this section, we present additional experimental results to supplement those presented in the main text. These results include extensive plots for the validity of Assumption~\ref{ass:incremental}, results for deep matrix sensing, ablation studies, and a discussion on the prevalence of low-rank updates for nonlinear networks.
Experiments were run using either a CPU with processor 3.0 GHz Intel Xeon Gold 6154 or a NVIDIA V100 GPU.
For clarity in notation throughout the Appendix, we provide a table of notation in Table~\ref{table:notation}.

\begin{table}[h!]
\begin{center}
\begin{tabular}{l|l}
\hline
 \multirow{1}{*}{Notation} & \multicolumn{1}{c}{Definition} \\
\hline
  $L \in \mathbb{N}$ & Depth of the DLN \\
\hline
  $\hat{r} \in \mathbb{N}$ & Estimated rank for the compressed DLN\\
\hline
  $\epsilon \in \mathbb{R}_+$ & Weight initialization scale \\
\hline
  $\eta \in \mathbb{R}_+$ & Learning rate \\
\hline
  $\alpha \in \mathbb{R}_+$ & Discrepant learning rate scale \\
\hline
  $\bm{W}_l \in \mathbb{R}^{d\times d}$ & $l$-th weight matrix of original DLN \\
\hline
  $\widetilde{\bm{W}}_l \in \mathbb{R}^{\hat{r}\times \hat{r}}$ & $l$-th weight matrix of compressed DLN \\
\hline
\end{tabular}
\end{center}
\caption{Summary of the notation used throughout this work.}
\label{table:notation}
\vspace{-8pt}
\end{table}

\subsection{Experimental Results on Incremental Learning}
\label{sec:incremental_app}

Here, we aim to present more empirical results to validate our observation of the incremental learning phenomenon. To that end, we synthetically generate a target matrix $\bM^* \in \R^{d\times d}$ with $d=100$ and rank $r=5$ and $r=10$ with $\hat{r} = 2r$. We consider deep matrix factorization, sensing, and completion, where our goal is to fit the target matrix $\bM^*$ using a DLN of $L=3$ with initialization scale $\epsilon = 10^{-3}$. To train the weights, we run GD with step size $\eta = 10$ and $\alpha = 5, 2, 5$ for matrix factorization, sensing, and completion, respectively. For deep matrix sensing, each sensing matrix $\bm{A}_i$ was filled with i.i.d. Gaussian entries $\mathcal{N}(0, 1)$, for all $i \in [m]$, where $m=2000$. For matrix completion, we consider the MCAR setting, with $20\%$ observed entries. The same setup was also used to generate Figure~\ref{fig:incre_mf_r5}. The results are displayed in Figures~\ref{fig:incre_mc_r5},~\ref{fig:incre_ms_r5}~\ref{fig:incre_mf_r10},~\ref{fig:incre_ms_r10},~\ref{fig:incre_mc_r10}. The same setup was also used to generate Figure~\ref{fig:incre_mf_r5}. In Figure~\ref{fig:incre_mf_r5_2d}, we display a two-dimensional plot of the change in singular values. These plots demonstrate that throughout deep matrix factorization and sensing problems, the principal components of the DLN are fitted incrementally, starting from the largest principal component to the next. The two-dimensional plot in Figure~\ref{fig:incre_mf_r5_2d} illustrates this more precisely. We can observe that the learning of the singular values occurs one at a time, while the other singular values remain close to their initial values until the previous one is adjusted to its target singular value. These findings serve to support our assumptions and observations as outlined in Assumption~\ref{ass:incremental}.

Furthermore, recall that in our definition of Assumption~\ref{ass:incremental}, there exist constants $c_{\text{val}}, c_{\text{vec}} \in [0, 1]$ that define the precision with which the singular values and vectors fit the target values, respectively. We perform a study on deep matrix factorization to demonstrate that both $c_{\text{val}}$ and $c_{\text{vec}}$ are (very) close to $0$. To this end, we synthetically generate a target matrix $\bM^* \in \R^{d\times d}$ with $d=100$ and rank $r=3$ with learning rate $\eta = 5$. We measure the difference in singular values and vectors throughout the course of GD, and display our results in Figure~\ref{fig:ass_cvals}.

\begin{figure}[h!]
    \centering
    \includegraphics[width=0.85\textwidth]{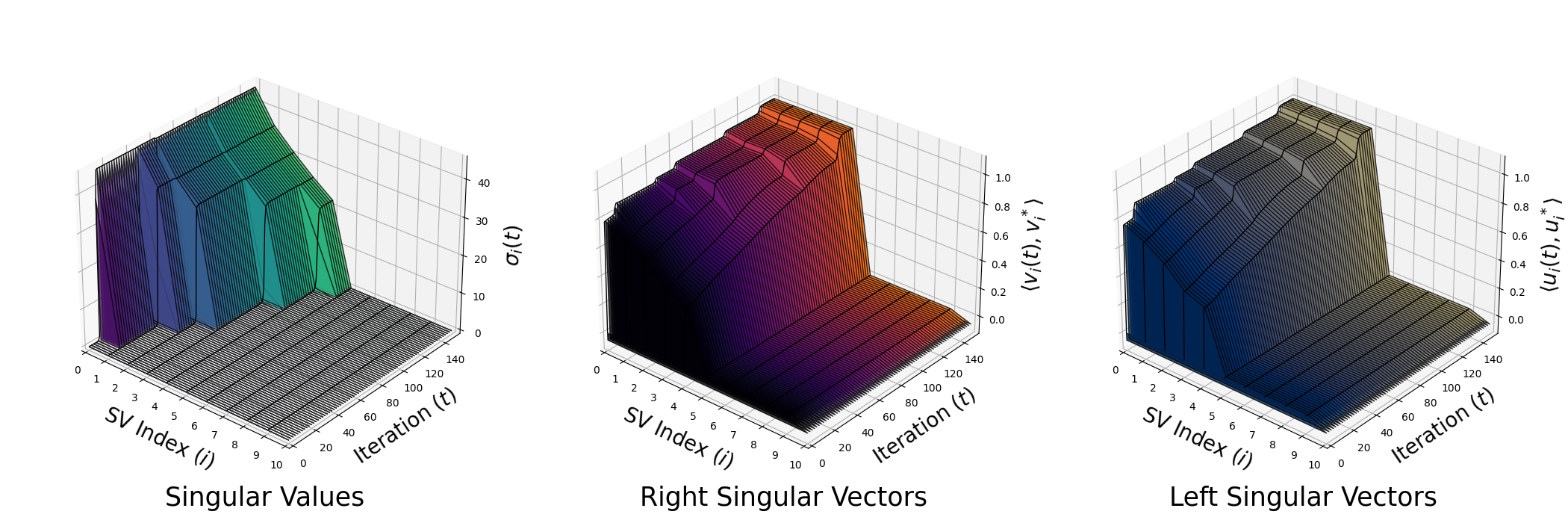}
    \caption{\textbf{Occurence of the incremental learning phenomenon in matrix completion.} We observe that the first $r=5$ singular values are fitted incrementally, along with their respective singular subspaces.}
    \label{fig:incre_mc_r5}
\end{figure}

\begin{figure}[h!]
    \centering
    \includegraphics[width=0.9\textwidth]{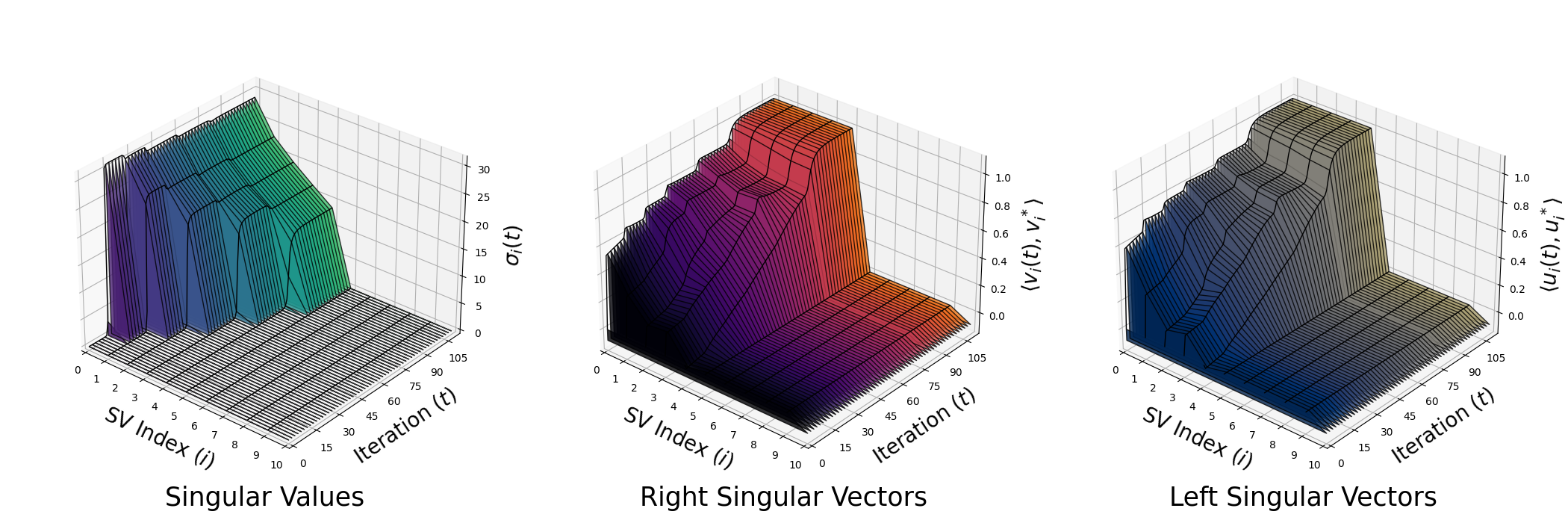}
    \caption{\textbf{Occurence of the incremental learning phenomenon in matrix sensing.} We observe that the first $r=5$ singular values are fitted incrementally, along with their respective singular subspaces.}
    \label{fig:incre_ms_r5}
\end{figure}

\begin{figure}[ht!]
    \centering
    \includegraphics[width=0.9\textwidth]{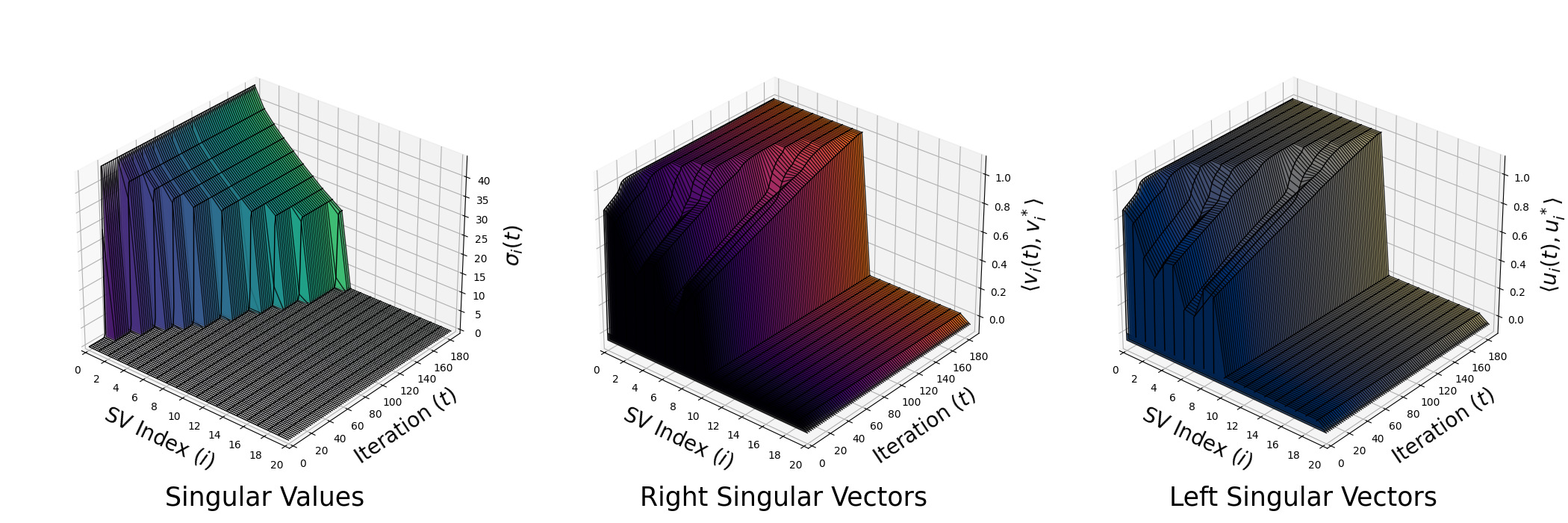}
    \caption{\textbf{Occurence of the incremental learning phenomenon in matrix factorization.} We observe that the first $r=10$ singular values are fitted incrementally, along with their respective singular subspaces.}
    \label{fig:incre_mf_r10}
\end{figure}

\begin{figure}[ht!]
    \centering
    \includegraphics[width=0.85\textwidth]{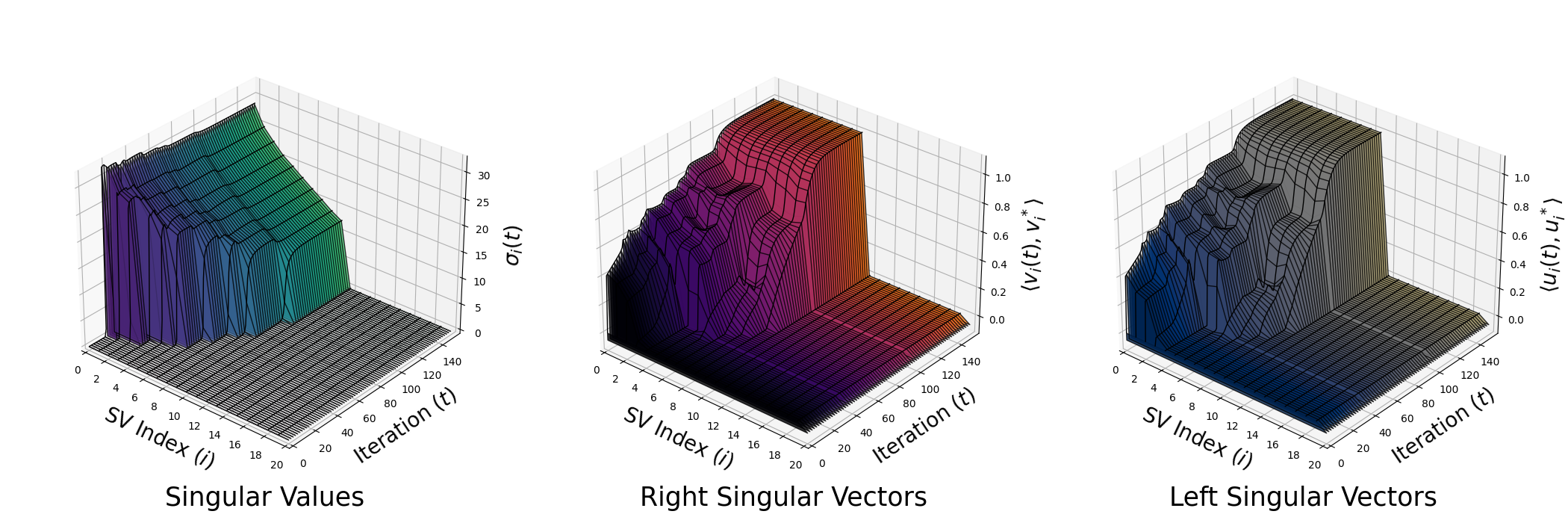}
    \caption{\textbf{Occurence of the incremental learning phenomenon in matrix sensing.} We observe that the first $r=10$ singular values are fitted incrementally, along with their respective singular subspaces.}
    \label{fig:incre_ms_r10}
\end{figure}

\begin{figure}[ht!]
    \centering
    \includegraphics[width=0.9\textwidth]{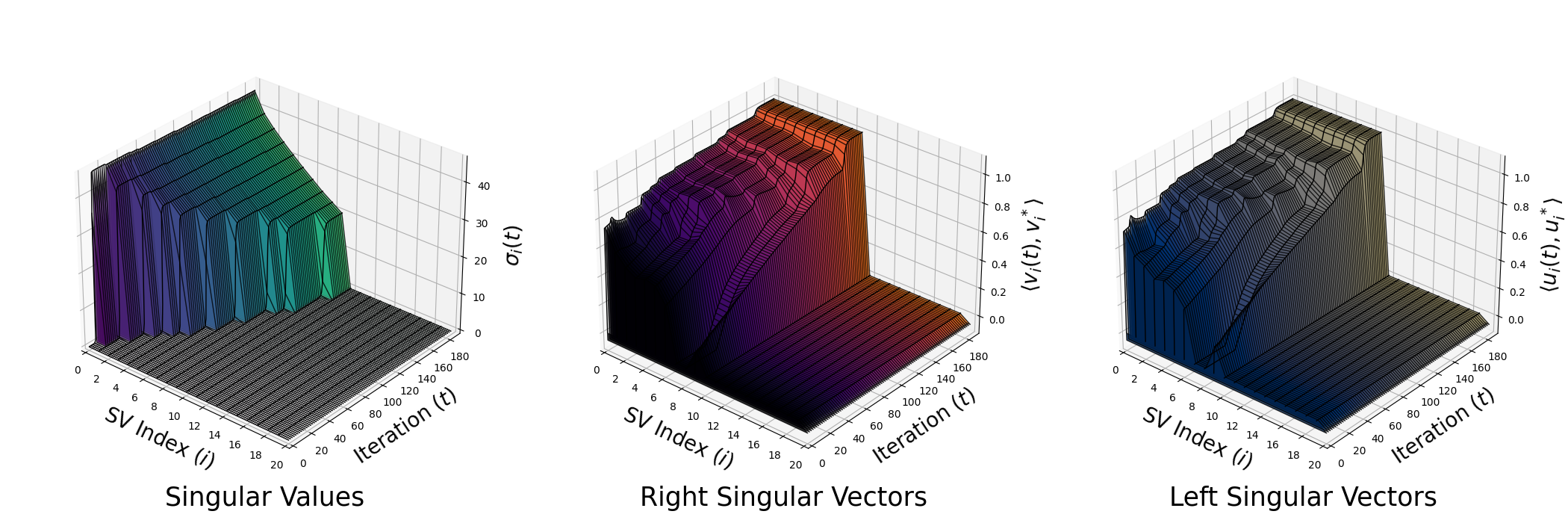}
    \caption{\textbf{Occurence of the incremental learning phenomenon in matrix completion.} We observe that the first $r=10$ singular values are fitted incrementally, along with their respective singular subspaces.}
    \label{fig:incre_mc_r10}
\end{figure}

\begin{figure}[ht!]
    \centering
    \includegraphics[width=0.45\textwidth]{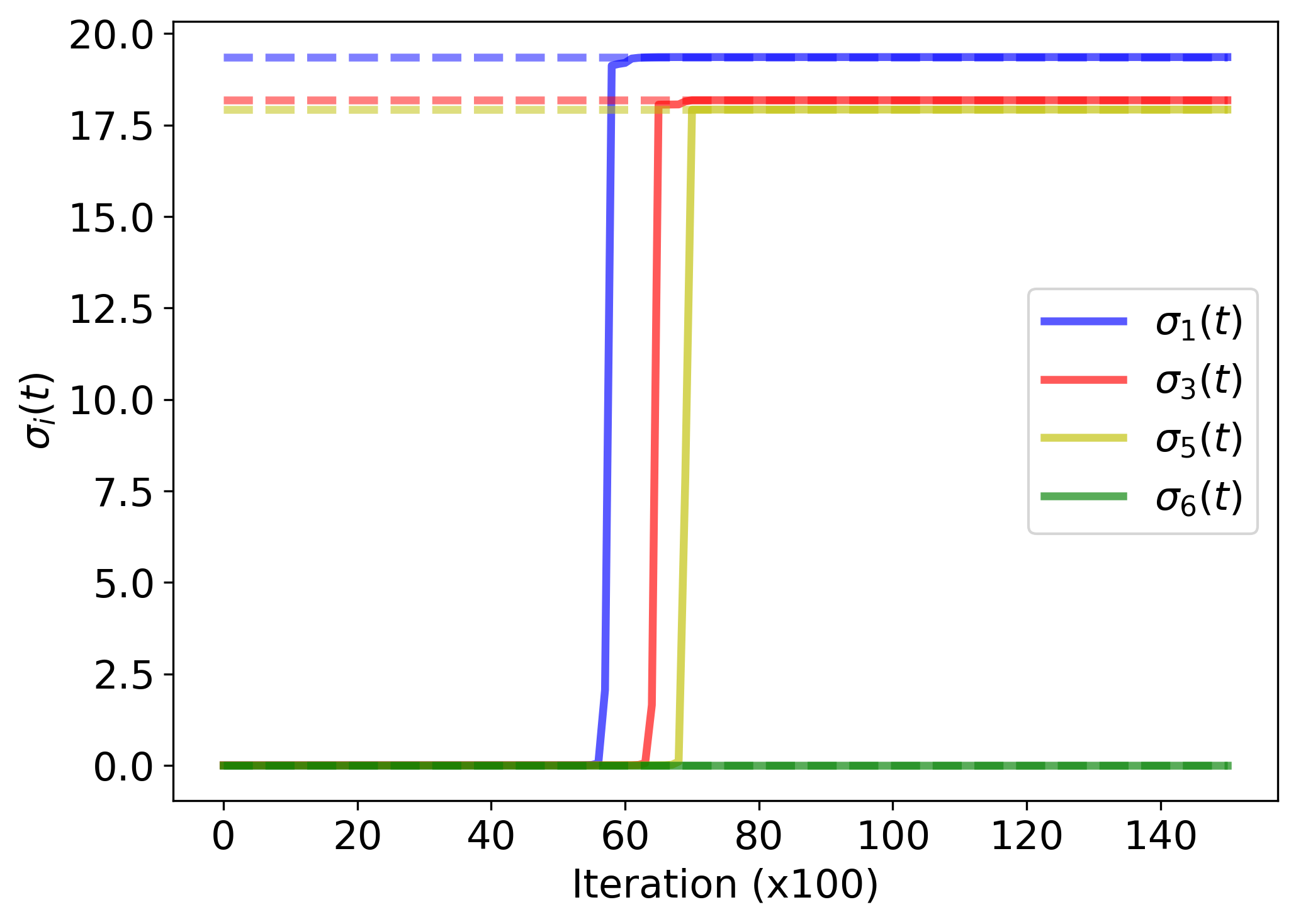}
    \caption{\textbf{Occurence of the incremental learning phenomenon in matrix factorization, with a more close up view on its singular values.} Since the underlying matrix is rank $r=5$, the sixth singular value and onwards stay (almost) unchanged from initialization. The dotted lines represent the magnitude of the target singular value ($\sigma_i^*$).}
    \label{fig:incre_mf_r5_2d}
\end{figure}

\begin{figure}[h!]
    \centering
    \includegraphics[width=\textwidth]{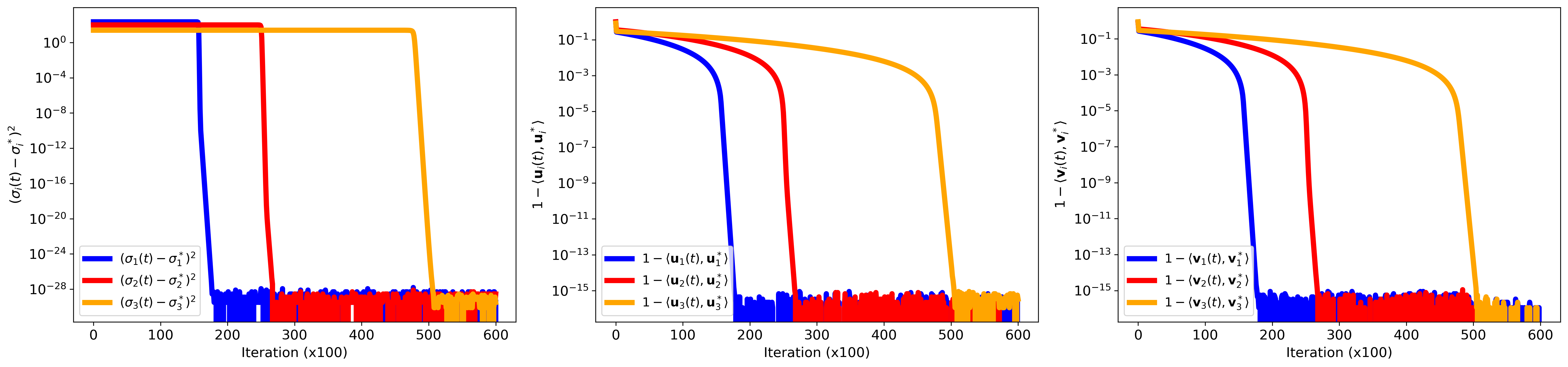}
    \caption{\textbf{Experiments observing the values of $c_{\text{val}}$ and $c_{\text{vec}}$ throughout the course of GD.} These plots demonstrate that after $t_i$, which indicate the time in which the $i$-th principal components are learned in Assumption~\ref{ass:incremental}, the values $c_{\text{val}}$ and $c_{\text{vec}}$ are very  close to $0$ (smaller than approximately $10^{-15}$).}
    \label{fig:ass_cvals}
\end{figure}


\subsection{Results on Deep Matrix Sensing}
\label{sec:deep_ms_app}

In this section, we present the deferred results for deep matrix sensing. We consider modeling a low-rank matrix $\bM^* \in \R^{d\times d}$ with varying dimensions $d$ and rank $r$. For the sensing operator, we use random Gaussian matrices $\bm{A}_i \in \R^{d\times d}$ with varying values of measurements $m$.
As previously discussed, to initialize $\widetilde{\bm{W}}_L(0)$ and $\widetilde{\bm{W}}_1(0)$, we take the left and right singular subspaces of the surrogate matrix
\begin{align*}
    \bm{S} = \mathcal{A}^{\dagger}\mathcal{A}(\bM^*).
\end{align*}
For the hyperparameters, we use the same setup as described in Section~\ref{sec:incremental_app}. The results are displayed in Figure~\ref{fig:ms_ms}. In Figure~\ref{fig:ms_ms}, we observe that for deep matrix sensing, we observe the same phenomenon, where the compressed network consistently outperforms the original network in terms of recovery error for all iterations of GD. This empirically shows the benefits of spectral initialization and the use of DLNs.

\begin{figure}[t!]
     \centering
     \begin{subfigure}{0.495\textwidth}
         \centering
         \includegraphics[width=\textwidth]{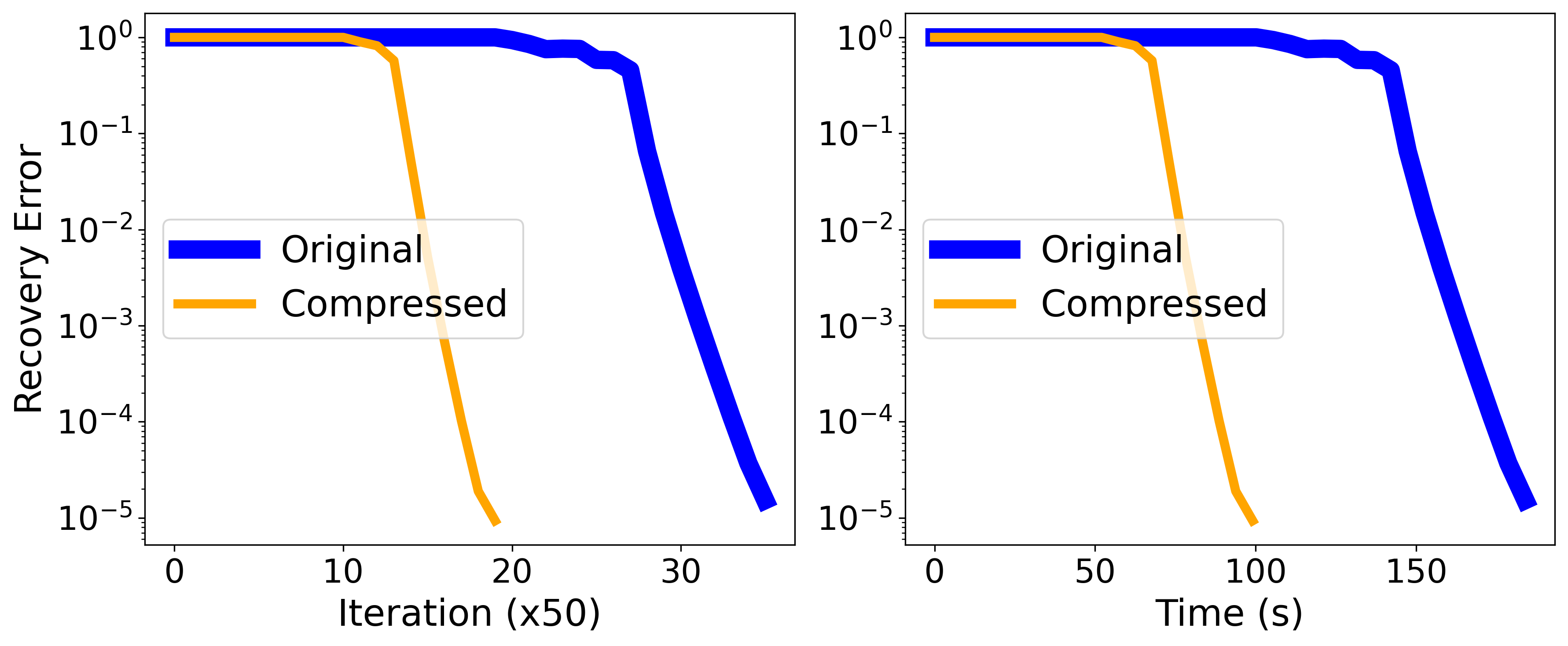}
         \caption*{$d=100$, $r=5$, $m=2000$}
     \end{subfigure}
     \begin{subfigure}{0.495\textwidth}
         \centering
         \includegraphics[width=\textwidth]{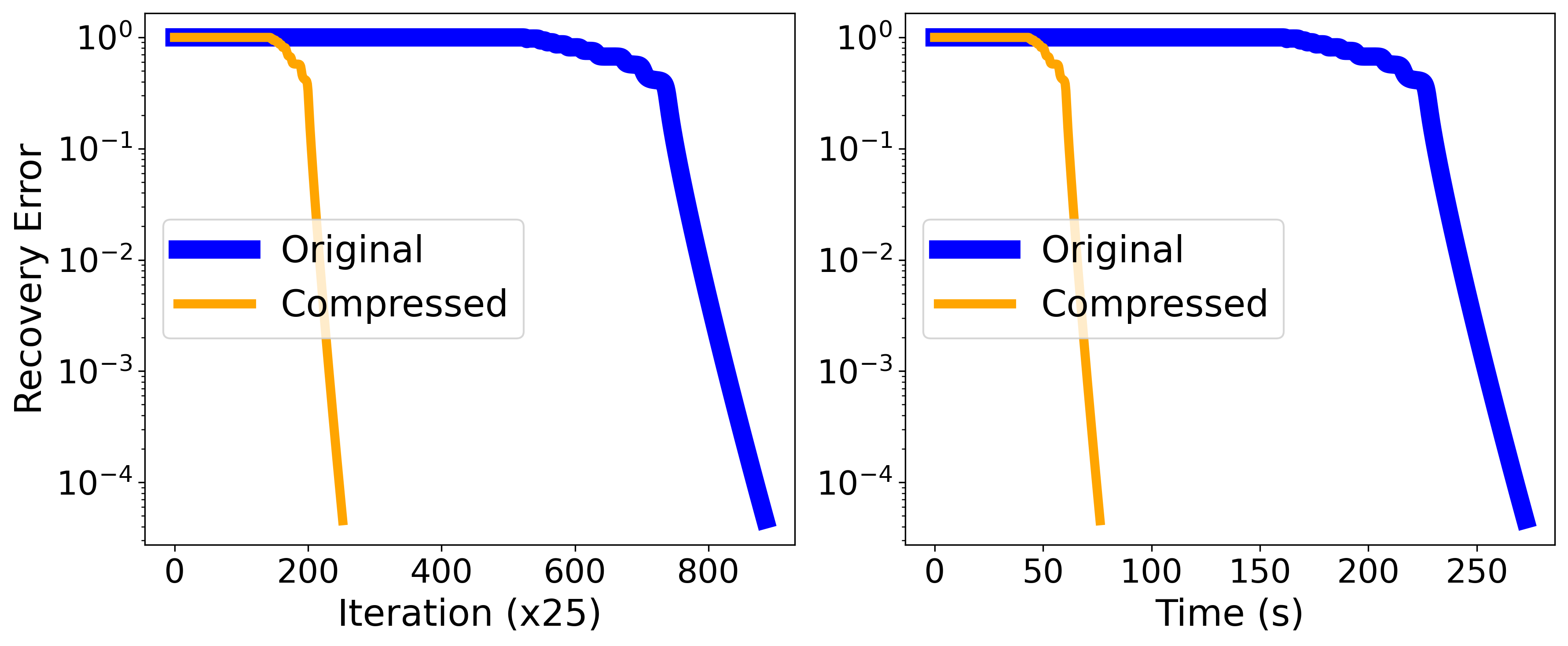}
         \caption*{$d=200$, $r=10$, $m=8000$}
     \end{subfigure}
     \caption{
     \textbf{Results on deep matrix sensing with random Gaussian measurements.} We observe the same phenomena in the matrix sensing setting similar to Figures~\ref{fig:mf_r5} and~\ref{fig:mc_synthetic}, where the compressed network consistently outperforms the original network.}
     \label{fig:ms_ms}
\end{figure}

\subsection{Ablation Studies}
\label{sec:ablation}

In this section, we conduct several ablation studies to demonstrate (1) the effect of the learning rate scale $\alpha$, (2) the performance of the compressed network with small \emph{random} initialization, and (3) the added time complexity of SVD.

\paragraph{The Effect of $\alpha$.}
In Section~\ref{sec:method}, we introduced the concept of a learning rate ratio (or scale), denoted as $\alpha$. Recall that in Figure~\ref{fig:pc_dist}, we showed that the factor matrices $\widetilde{\bm{W}}_L$ and $\widetilde{\bm{W}}_1$ ultimately align the the singular subspaces of the target matrix. By increasing $\alpha$ (i.e., by choosing $\alpha > 1$), we can hope to fit these target subspaces more quickly, leading to accelerated convergece in terms of iterations. To this end, we peform an ablation study where we vary alpha from $\alpha \in [0.5, 10]$ in increments of $0.25$, and report the recovery error on synthetic matrix completion. We present our results in Figure~\ref{fig:dlr_ablation}, where we observe for larger values of $\alpha$ significantly improves convergence. For values of $\alpha \leq 1$, we have slower convergence in terms of iterations than the original wide DLN, but since each iteration is much faster, it still takes much less time to train the compressed DLN. However, for large values of $\alpha$ (e.g. $\alpha \gg 10$), the training becomes unstable and often diverges. In Figure~\ref{fig:dlr_ablation_movielens}, we demonstrate a case of this, where for large values of $\alpha$, the compressed network overfits and incurs a large recovery error. For smaller values of $\alpha$ (e.g., $\alpha=0.1$), the networks takes significantly longer to converge.

\begin{figure}[h!]
    \centering
    \includegraphics[width=0.75\textwidth]{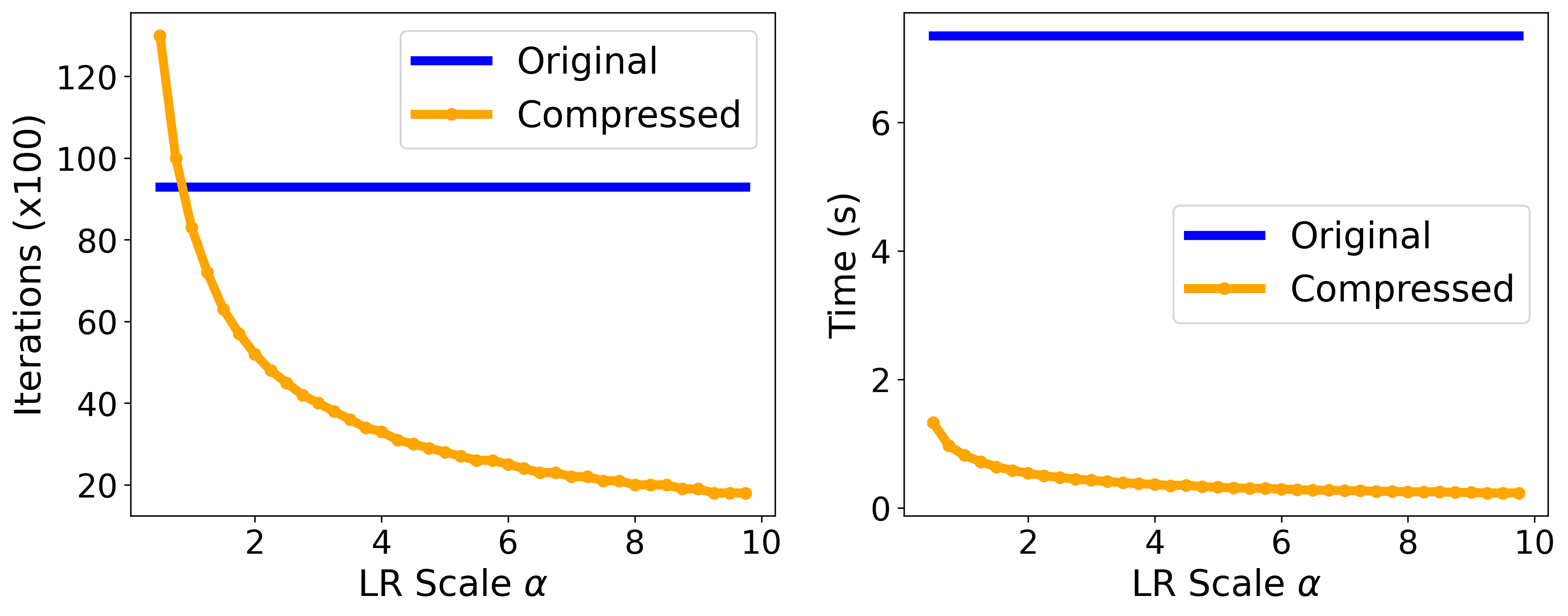}
    \caption{\textbf{Ablation study on the performance of the compressed DLN with varying values of $\alpha$.} For small values of $\alpha$, the compressed network takes more iterations to reach the convergence criteria, but since the time complexity of each iteration is smaller, it still takes much shorter time overall.}
    \label{fig:dlr_ablation}
\end{figure}

\begin{figure}[h!]
    \centering
    \includegraphics[width=0.75\textwidth]{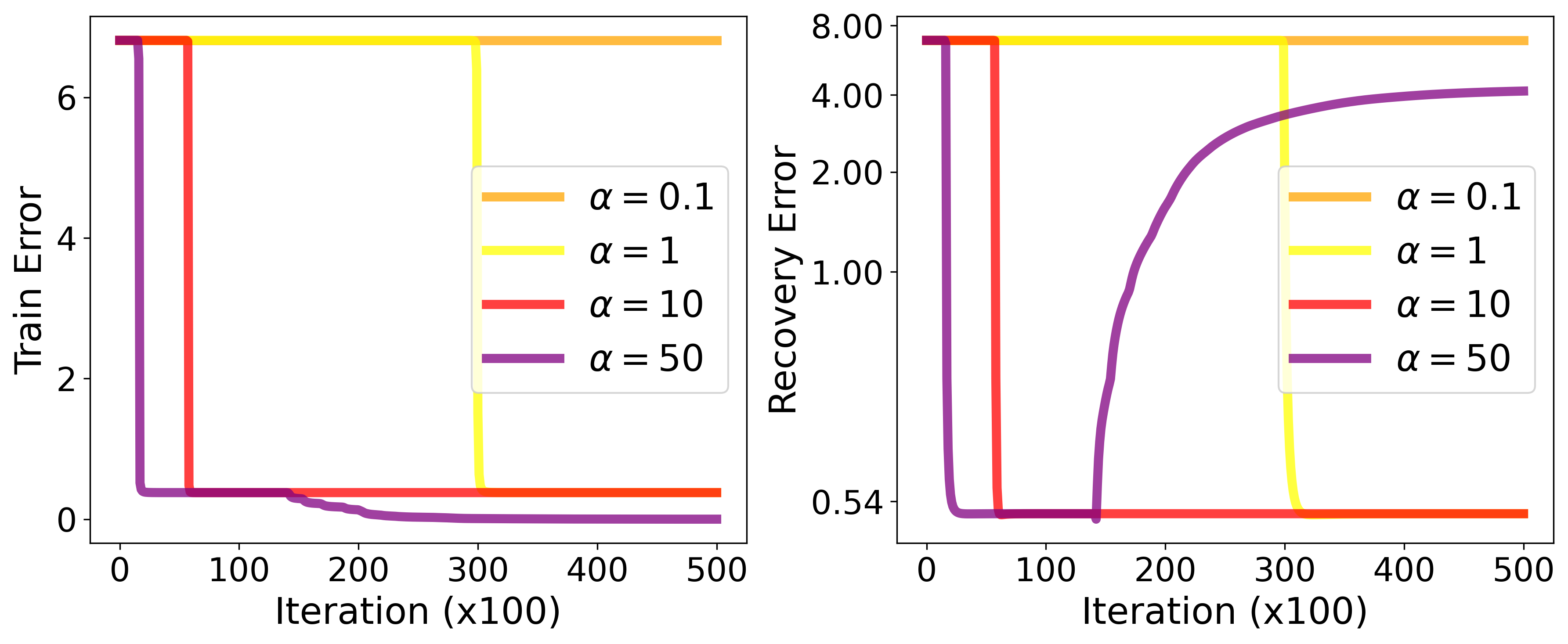}
    \caption{\textbf{Ablation study on $\alpha$ on the MovieLens dataset.} These plots demonstrate that for large values of $\alpha$, the compressed network can overfit, whereas smaller values of $\alpha$ increases the time it takes for the model to converge.}
    \label{fig:dlr_ablation_movielens}
\end{figure}

\paragraph{The Use of Small Random Initialization.} 

Thus far, all of our experiments and theory relied on the fact that the original network was initialized with orthogonal weights scaled by a small scale $\epsilon > 0$. In this section, we perform a study to show that our compressed network also outperforms the original network when the original network is initialized with random weights also scaled by $\epsilon$. We consider deep matrix factorization and matrix completion, following the setting of Section~\ref{sec:lowrank_exp}. We showcase our results in Figure~\ref{fig:rand_init}, where show that the same phenomenon persists in this setting as well.

\begin{figure}[h!]
     \centering
     \begin{subfigure}{0.495\textwidth}
         \centering
         \includegraphics[width=\textwidth]{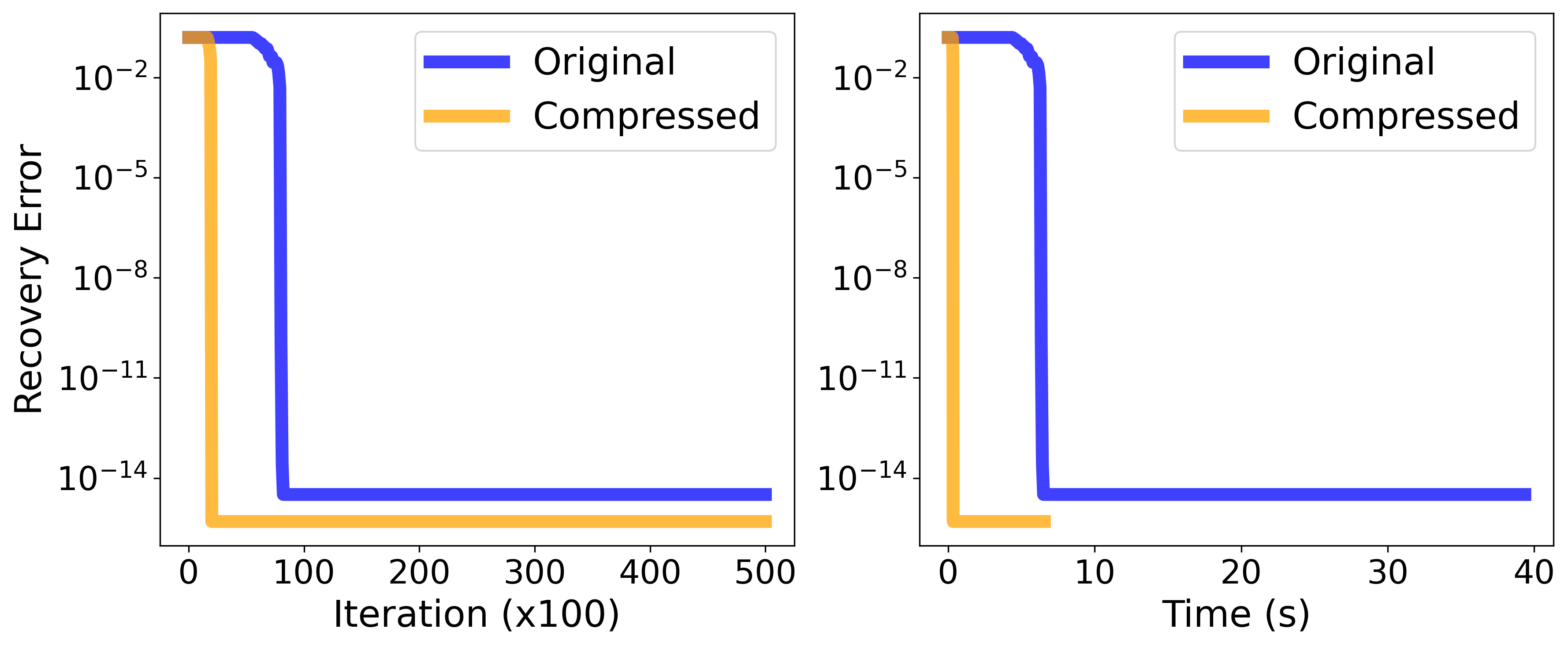}
         \caption*{Deep Matrix Factorization}
     \end{subfigure}
     \begin{subfigure}{0.495\textwidth}
         \centering
         \includegraphics[width=\textwidth]{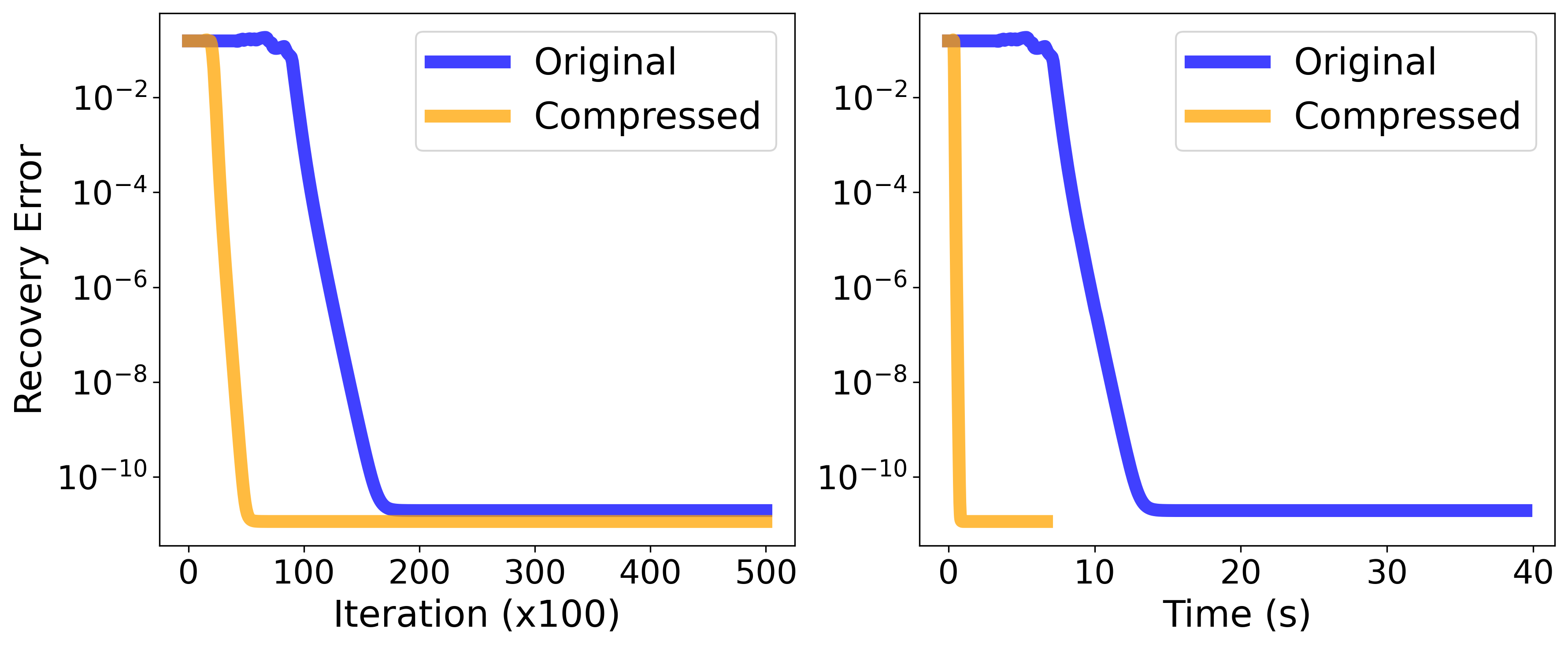}
         \caption*{Deep Matrix Completion}
     \end{subfigure}
     \caption{\textbf{Comparison of the performance of networks when the original network is initialized with near-zero random weights.} We observe that the compressed network still outperforms the original wide network in this setting for both deep matrix factorization and completion.}
     \label{fig:rand_init}
\end{figure}

\paragraph{Added Time Complexity of Taking SVD.}
Recall that our algorithm involves taking the SVD of a carefully constructed surrogate matrix to initialize the weights of the compressed network. In this study, we report the time (in seconds) it takes to train both the original network and the compressed network, including the time taken for the SVD step, for varying values of $d$ (the dimensions of the target matrix). Our results are presented in Table~\ref{table:svd_time}. We observe that the time taken for the SVD is almost negligible compared to the training time for synthetic data. For real larger-scale experiments, while we anticipate that computing the SVD may take a longer time, it will also significantly reduce the training time, resulting in a smaller overall time.
\begin{table}[h!]
\begin{center}
\caption{\textbf{The time (measured in seconds) required to train both the compressed and original networks, accounting for the SVD step.} We note that the time taken to compute the SVD is nearly negligible.}
\label{table:svd_time}
\begin{tabular}{ c|c c |c } 
\hline
Method & SVD & Training & Total \\
\hline
DLN ($d=100$) & - & 184.3 & 184.3 \\ 
C-DLN ($d=100$)& 0.009 & 109.7 & 109.7\\ 
\hline
DLN ($d=300$) & - & 511.2 & 511.2 \\ 
C-DLN ($d=300$)& 0.082 & 381.3 & 381.4 \\ 
\hline
\end{tabular}
\end{center}
\end{table}

\paragraph{The Effect of Depth.} In this section, we make a quick note on the effect of depth on matrix completion. We synthetically generate data according to Section~\ref{sec:lowrank_exp}.
Here, our objective is to answer whether the compressed network can still outperform the original network when the number of layers is small (e.g. $L=2$). The work by Arora et al.~\cite{arora2019implicit} suggests that the implicit regularization property is ``stronger'' as a function of depth, where deeper networks favor more low-rank solutions.
In Figure~\ref{fig:depth_effect}, we illustrate that when the number of observed measurements is very limited and the number of layers is small, both networks seem to perform considerably worse compared to the 3-layer case. While the training error is still smaller for the compressed DLN, the recovery error is slightly higher. This result implies that deeper networks are more favorable than shallow ones when the number of measurements is extremely limited, and the advantages of the compressed network become especially apparent when both networks can achieve a low recovery error.
\begin{figure}[h!]
    \centering
    \includegraphics[width=0.8\textwidth]{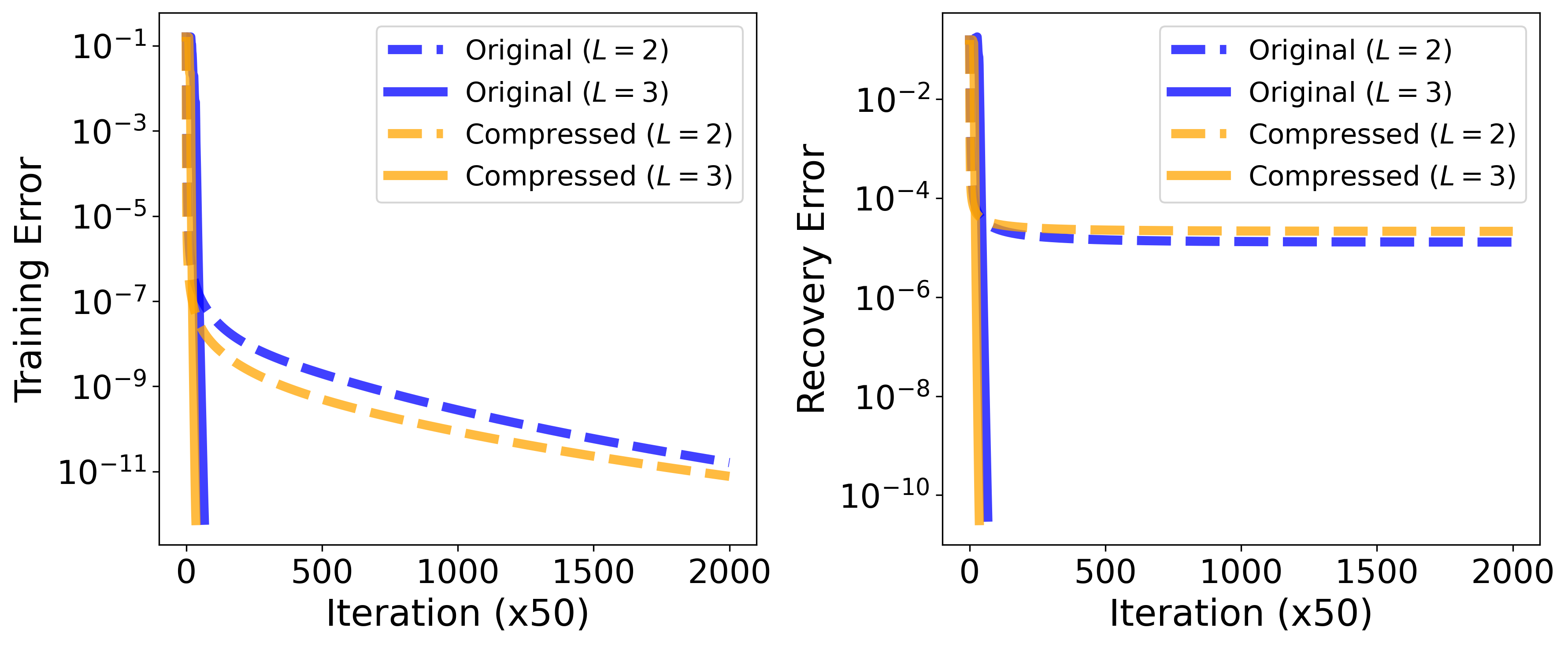}
    \caption{\textbf{The effect of depth on deep matrix completion with limited number of measurements.} We show that the compressed network only has advantages over the original network when the original network converges in terms of recovery (or test) error.}
    \label{fig:depth_effect}
\end{figure}

\paragraph{Choosing $\hat{r}$.}
Here, we would like to make a brief remark on choosing $\hat{r}$. Throughout our experiments thus far, we chose $\hat{r} = 2r$ to demonstrate that even when the rank is overspecified, one can obtain an acceleration in training time with significantly reduced recovery error.
For choosing $\hat{r}$ in practice, notice that for matrix recovery problems, Theorem~\ref{thm:recovery_mf} states that even with a choice of $\hat{r} = d$, the recovery error of the ``compressed'' network will have smaller recovery error than the original network due to the choice of initialization. This implies that if one knows beforehand that the rank of the target matrix is much smaller than its ambient dimension (i.e., $r \ll d$), then we can simply choose $\hat{r} = \frac{3}{4}d$ or $\hat{r} = \frac{1}{2} d$ and observe the benefits of our network as seen in Section~\ref{sec:experiments}.

\subsection{Additional Results and Details for Nonlinear Networks}
\label{sec:training_dets}

\paragraph{Additional Results.}
In Figure~\ref{fig:plot_svals}, we plotted the singular values of the change in weight updates from initialization of the penultimate layer matrix (i.e., $\bm{W}(t) - \bm{W}(0)$ where $\bm{W}$ denotes the penultimate layer matrix) and showed that the changes happen within a low-dimensional subspace. In Figure~\ref{fig:subspace_dist}, we plot the change in the top-$r$ singular vectors. We use the subspace distance defined as 
\begin{align}
\label{eq:subspace_dist}
    \text{Subspace Distance  } = r-\|\bm{U}_r(t)^\top\bm{U}_r(t+1)\|_{\mathsf{F}}^2,
\end{align}
where $\bm{U}_r$ denotes the singular vectors of $\bm{W}$ corresponding to the top-$r$ singular values. For this experiment, we choose $r=10$ and initialize the networks as described in the training details (next subsection). Here, we observe that the subspace distance is very small throughout training all network architectures, which hints that the subspaces greatly overlap throughout all $t$ (i.e., the training occurs within an invariant subspace). 


In Figure~\ref{fig:mnist_cifar}, we present additional plots for experiments on deep nonlinear networks. Here, we depict the change in test accuracy over GD iterations for MLP and ViT trained on FashionMNIST and CIFAR-10, respectively. We display the test accuracies for the original (in green), original network whose penultimate layer is a DLN (Original + DLN in blue), and the original network whose penultimate layer is a compressed DLN (in orange) throughout GD and their respective training times in seconds.
We observe that networks whose penultimate layer is modeled as a DLN exhibit the best performance in terms of test accuracy, but they take the longest time to train due to the additional overparameterization. By using our compression approach, we can reduce this training time by almost $2\times$, while achieving very similar test accuracy. Interestingly, Figure~\ref{fig:mnist_cifar} shows that for ViT, we can outperform the other networks using our compressed DLN, while having the shortest training time.

\begin{figure}
    \centering
    \includegraphics[width=\textwidth]{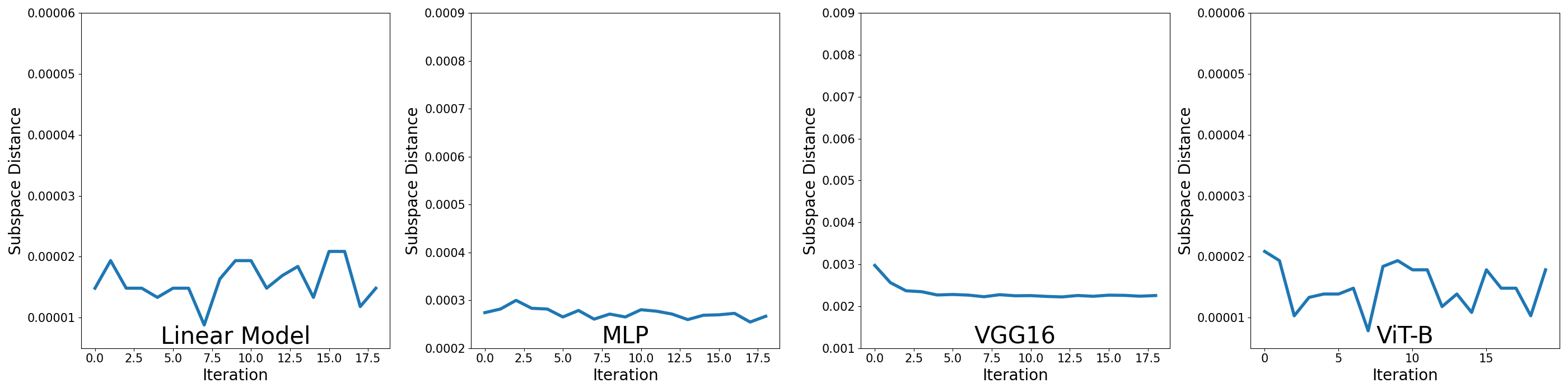}
    \caption{\textbf{Distance between consecutive singular subspaces of the penultimate weight matrices.} These plots show that the subspace distance defined in Equation~(\ref{eq:subspace_dist}) is small, hinting that the training happens within an invariant subspace.}
    \label{fig:subspace_dist}
\end{figure}

\begin{figure}[h!]
     \centering
     \begin{subfigure}{0.495\textwidth}
         \centering
         \includegraphics[width=\textwidth]{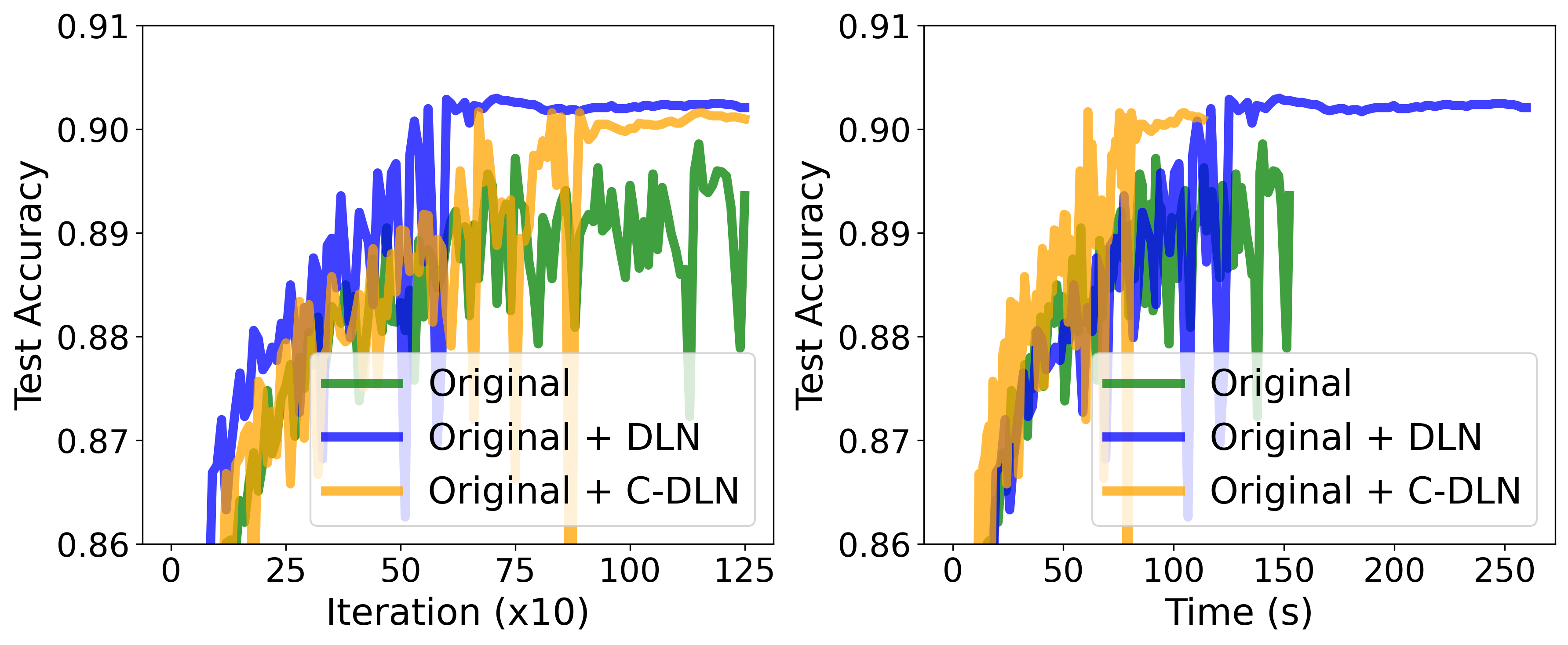}
         \caption*{Results on FashionMNIST}
     \end{subfigure}
     \begin{subfigure}{0.495\textwidth}
         \centering
         \includegraphics[width=\textwidth]{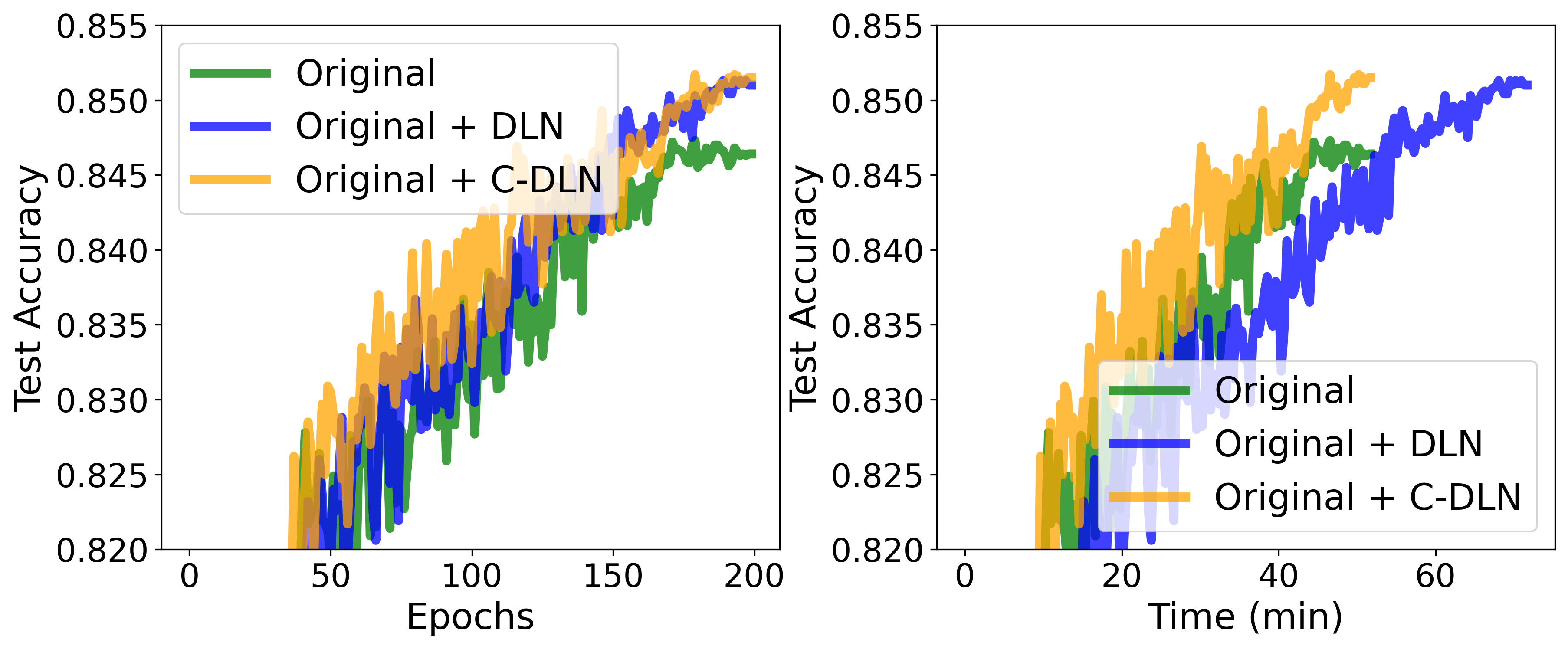}
         \caption*{Results on CIFAR-10}
     \end{subfigure}

     \caption{
     \textbf{Results on classification tasks with nonlinear networks with overparameterized penultimate layers.}
     Left: Results with FashionMNIST dataset with MLP. Right: Results on CIFAR-10 dataset with ViT. For both experiments, we observe that the compressed networks (colored in orange) has similar test accuracy (if not better) than the original network (colored in blue) throughout training, while reducing run-time by almost $2\times$. The plot colored in green represents the network with a standard penultimate layer, which has lower test accuracy throughout all experiments.}
     \label{fig:mnist_cifar}
\end{figure}

\paragraph{Training Details.} To plot Figure~\ref{fig:plot_svals}, we observe the change in singular values of the penultimate weight matrix for DLNs, MLPs, VGG~\cite{DBLP:journals/corr/SimonyanZ14a}, and ViT-B~\cite{DBLP:conf/iclr/DosovitskiyB0WZ21}. We train the DLN starting from orthogonal initialization using SGD with learning rate $0.01$ and batch size $128$. All of the nonlinear networks were initialized using random uniform initialization.
For MLP, we train using the same parameters as the DLN. For VGG, we use batch size $128$ with learning rate $0.05$, weight decay parameter $5 \times 10^{-4}$, momentum $0.9$ and step size scheduler with cosine annealing. Lastly, for ViT-B, we train using ADAM with a batch size of $512$ and learning rate $10^{-5}$ with cosine annealing. The ViT-B architecture is the same as the one presented by Dosovitskiy et al.~\cite{DBLP:conf/iclr/DosovitskiyB0WZ21}.

%% file: contents/new_proofs.tex
\section{Theoretical Results}
\label{sec:proofs_app}

\subsection{Deferred Proofs for Spectral Initialization}
To keep this section self-contained, we will first restate all of our results. Throughout this section, we will consider the compressed deep matrix factorization problem, which at GD iterate $t$ is defined as
\begin{align*}
    \ell(\widetilde{\bm{\Theta}}(t)) = \frac{1}{2} \|\widetilde{\bm{W}}_{L:1}(t) -\bm{M}^*\|^2_{\mathsf{F}},
\end{align*}
where $\widetilde{\bm{W}}_L(0) = \epsilon \cdot \bm{U}^*_{\hat{r}}$, $\widetilde{\bm{W}}_1(0) = \epsilon \cdot \bm{V}^{*\top}_{\hat{r}}$, and $\widetilde{\bm{W}}_l(0) =  \epsilon\cdot \bm{I}_{\hat{r}}$ for $2 \leq l \leq L-1$,
for some small initialization scale $\epsilon > 0$.

\begin{manualtheoreminner}
Let $\bM^* \in \R^{d\times d}$ be a matrix of rank $r$ and $\bM^* = \bm{U}^* \bm{\Sigma}^* \bm{V}^{*\top}$ be a SVD of $\bM^*$. Suppose we run Algorithm~\ref{alg:alg} to update all weights $\left( \widetilde{\bm{W}}_l \right)_{l=1}^L$ of Equation~(\ref{eq:comp_deep_mf}), where $\mathcal{A} = \text{Id}$. Then, the end-to-end compressed DLN possesses low-dimensional structures, in the sense that for all $t\geq 1$,  $\widetilde{\bm{W}}_{L:1}(t)$ admits the following decomposition:
\begin{align}
    \widetilde{\bm{W}}_{L:1}(t) = \bm{U}^*_{\hat{r}} 
    \begin{bmatrix}
    \bm{\Lambda}(t)  &\bm{0} \\
        \bm{0} & \beta(t)^L\cdot\bm{I}_{\hat{r} - r}
    \end{bmatrix}
    \bm{V}_{\hat{r}}^{*\top},
\end{align}
where $\bm{\Lambda}(t) \in \R^{r\times r}$ is a diagonal matrix with entries $\lambda_i(t)^L$, where
\begin{align}
    \lambda_i(t) = \lambda_i(t-1) \cdot \left(1 - \eta \cdot (\lambda_i(t-1)^L - \sigma_i^*) \cdot \lambda_i(t-1)^{L-2} \right), \quad 1 \leq i \leq r,
\end{align}
with $\lambda_i(0) = \epsilon$ and $\sigma^*_i$ is the $i$-th diagonal entry of $\bm{\Sigma}^*$ and
\begin{align}
    \beta(t) = \beta(t-1) \cdot \left( 1 - \eta \cdot \beta(t-1)^{2(L-1)} \right), 
\end{align}
with $\beta(0) = \epsilon$.
\end{manualtheoreminner}

\begin{proof}
    We will prove using mathematical induction. 
    
    \paragraph{Base Case.} For the base case at $t=0$, the decomposition holds immediately by the choice of initialization:
    \begin{align*}
        \widetilde{\bm{W}}_{L:1}(0) &= \underbrace{\bm{U}^*_{\hat{r}} \cdot \epsilon \bm{I}_{\hat{r}}}_{\widetilde{\bm{W}}_L(0)} \cdot \underbrace{\epsilon^{(L-2)} \bm{I}_{\hat{r}}}_{\widetilde{\bm{W}}_{2:L-1}(0)} \cdot \underbrace{\epsilon \bm{I}_{\hat{r}} \cdot \bm{V}^{*\top}_{\hat{r}}}_{\widetilde{\bm{W}}_{1}(0)} \\
        &= \bm{U}^*_{\hat{r}} \cdot \begin{bmatrix}
            \epsilon^L \bm{I}_r & \bm{0} \\
            \bm{0} & \epsilon^L \bm{I}_{\hat{r} - r}
        \end{bmatrix}
        \cdot \bm{V}^{*\top}_{\hat{r}} \\
        &=\bm{U}^*_{\hat{r}} \cdot \begin{bmatrix}
            \bm{\Lambda}(0) & \bm{0} \\
            \bm{0} & \beta(0)^L \bm{I}_{\hat{r} - r}
        \end{bmatrix}
        \cdot \bm{V}^{*\top}_{\hat{r}}.
    \end{align*}

\paragraph{Inductive Hypothesis.}
Now, by the inductive hypothesis, suppose that the decomposition holds for some $t \geq 0$. Notice that by the rotational invariance of the Frobenius norm, we can rewrite the objective function into
\begin{align}
    \ell(\widetilde{\bm{\Theta}}(t)) &= \frac{1}{2} \|\widetilde{\bm{W}}_{L:1}(t) -\bm{M}^*\|^2_{\mathsf{F}} \\
    &= \frac{1}{2}  \left\|\bm{U}^*_{\hat{r}}\begin{bmatrix}
       \bm{\Lambda}(t) & \bm{0} \\
        \bm{0} & \beta(t)^L\bm{I}_{\hat{r} - r}
    \end{bmatrix} \bm{V}^{*\top}_{\hat{r}} -  \bm{U}^*_{\hat{r}}\begin{bmatrix}
        \bm{\Sigma}_{r}^* & \bm{0} \\
        \bm{0} & \bm{0}
    \end{bmatrix} \bm{V}^{*\top}_{\hat{r}}\right\|^2_{\mathsf{F}} \\
        \label{eq:compressed_update}
    &= \frac{1}{2}  \left\|\begin{bmatrix}
       \bm{\Lambda}(t) & \bm{0} \\
        \bm{0} & \beta(t)^L\bm{I}_{\hat{r} - r}
    \end{bmatrix} -\begin{bmatrix}
        \bm{\Sigma}_{r}^* & \bm{0} \\
        \bm{0} & \bm{0}
    \end{bmatrix} \right\|^2_{\mathsf{F}}
\end{align}
where $\bm{\Sigma}_{r}^*$ is the leading $r$ principal submatrix $\bm{\Sigma}^*$. Then, Equation~(\ref{eq:compressed_update}) implies that updating the factor $\widetilde{\bm{W}}_L(t)$ is equivalent to updating a diagonal matrix of dimensions $\hat{r}$ and multiplying by $\bm{U}^*_{\hat{r}}$ (and respectively $\widetilde{\bm{W}}_1(t)$ with $\bm{V}^{*\top}_{\hat{r}}$). 
To this end, we will consider updating the diagonal entries to show that the decomposition holds for $t+1$. 

\paragraph{Inductive Step.} We will first consider the top-$r$ components. 
Notice that the $i$-th diagonal entry $\forall i\in [r]$ of the $l$-th layer matrix $\forall l \in [L]$ are all scalars, and so the update steps for the $i$-th entry are given by 
\begin{align*}
    \lambda_i(t+1) &= \lambda_i(t) - \eta \cdot\left(\lambda_i(t)^L - \sigma_i^*\right) \cdot \lambda_i(t)^{L-1} \\
    &= \lambda_i(t) \cdot \left(1 - \eta \cdot (\lambda_i(t)^{L} - \sigma_i^*)   \cdot \lambda(t)_i^{L-2} \right).
\end{align*}
Thus, the $i$-th entry of $\bm{\Lambda}(t)$ is given by $\lambda_i(t)^L$ for $1 \leq i \leq r$.
Similarly, the update steps for the last $\hat{r} - r$ elements are given by
\begin{align*}
    \beta(t+1) &= \beta(t) - \eta \cdot \left(\beta(t)^L \cdot \beta(t)^{L-1} \right) \\
    &= \beta(t) \cdot \left(1 - \eta \cdot \beta(t)^{2(L-1)} \right),
\end{align*}
which gives us $\beta(t)^L$ for the last $\hat{r} - r$ elements.
This completes the proof. 
\end{proof}

\begin{manualcorollaryinner}
    Let $\bm{W}_{L:1}(0)$ denote the original DLN at $t=0$ initialized with orthogonal weights according to Equation~(\ref{eq:orth_init}) and $\widetilde{\bm{W}}_{L:1}(0)$ denote the compressed DLN at $t=0$. Then, we have
\begin{align}
        \| \bm{W}_{L:1}(0) - \bM^*\|^2_{\mathsf{F}} \geq  \| \widetilde{\bm{W}}_{L:1}(0) - \bM^*\|^2_{\mathsf{F}}.
\end{align}
\end{manualcorollaryinner}
\begin{proof}
    Let $\bm{W}_{L:1}(0) = \bm{U}_L \bm{\Sigma}_{L:1}\bm{V}_1^{\top}$ and $\bM^* = \bm{U}^* \bm{\Sigma}^* \bm{V}^{*\top}$ denote the SVD of $\bm{W}_{L:1}(0)$ and $\bM^*$, respectively. Since $\bm{W}_{L:1}(0)$ is a product of orthogonal matrices scaled by $\epsilon$, its singular values are $\bm{\Sigma}_{L:1} = \epsilon^L \bm{I}_d$. Then, we have
    \begin{align*}
        \|\bm{W}_{L:1}(0) - \bM^*\|^2_{\mathsf{F}} &\geq \|\bm{\Sigma}_{L:1} - \bM^*\|^2_{\mathsf{F}} \\
        &= \|\epsilon^L \bm{I}_d - \bm{\Sigma}^*\|^2_{\mathsf{F}} \tag{By Lemma~\ref{lem:von}}\\
        &\geq \|\epsilon^L \bm{I}_{\hat{r}} - \bm{\Sigma}_{\hat{r}}^*\|^2_{\mathsf{F}} \tag{$\hat{r} \leq d$} \\
        &= \| \widetilde{\bm{W}}_{L:1}(0) - \bM^*\|^2_{\mathsf{F}} \tag{Theorem~\ref{thm:parsimony} at $t=0$}.
    \end{align*}
This proves the result.
\end{proof}

\subsection{Deferred Proofs for Incremental Learning}

In this section, we provide the deferred proofs from Section~\ref{sec:benefits_inc}. Recall that in this analysis, we will consider running gradient flow over the original network in Equation~(\ref{eq:deep_mf}) and over the compressed network in Equation~(\ref{eq:comp_deep_mf}).
Before we present our proof for our main result, we first restate a result from Arora et al.~\cite{arora2019implicit}, which characterizes change in singular values of the original network $\bm{W}_{L:1}(t)$.

\begin{proposition}[Arora et al.~\cite{arora2019implicit}]
\label{thm:arora}
    Consider running gradient flow over the deep matrix factorization problem defined in Equation~(\ref{eq:dln_setup}) with $\mathcal{A} = Id$ and $L\geq 2$. The end-to-end weight matrix can be expressed as 
\begin{align}
    \bm{W}_{L:1}(t) = \bm{U}(t)\bm{S}(t)\bm{V}^{\top}(t),
\end{align}
where $\bm{U}(t)$ and $\bm{V}(t)$ are orthonormal matrices and $\bm{S}(t)$ is a diagonal matrix. Denote $\sigma_i(t)$ as the $i$-th entry of the matrix $\bm{S}(t)$. Then, each
$\sigma_i(t)$ for all $i=1,\dots, d$ has the following gradient flow trajectory:
\begin{align}
    \dot{\sigma}_{i}(t) = -L \cdot (\sigma^2_i(t))^{1-1/L} \cdot \langle \nabla_{\bm{W}_{L:1}}\ell(\bm{\Theta}(t)), \bm{u}_i(t)\bm{v}_i^{\top}(t)  \rangle, \label{eq.dynamic_weihu}
\end{align}
where $\bm{u}_i(t)$ and $\bm{u}_i(t)$ denotes the $i$-th column of $\bm{U}(t)$ and $\bm{V}(t)$ respectively and
\begin{align}
    \nabla_{\bm{W}_{L:1}} \ell(\bm{\Theta}) = \bm{W}_{L:1} - \bm{M}^*.
\end{align}
\end{proposition}
We would like to remark that in order to directly use this result, we need to assume that the factor matrices are balanced at initialization, i.e.,
\begin{align*}
    \bm{W}^{\top}_{l+1}(0)\bm{W}_{l+1}(0) = \bm{W}_{l}(0)\bm{W}_{l}^{\top}(0), \quad l = 1,\ldots, L-1.
\end{align*}
Since we are using orthogonal initialization scaled with a small constant $\epsilon > 0$, we immediately satisfy this condition.
Equipped with this result, we introduce our main result.

\begin{manualtheoreminner}
    Let $\bM^* \in \R^{d\times d}$ be a rank-$r$ matrix and such that $\hat{r} \in \N$ with $\hat{r} \geq r$. Suppose that we run gradient flow with respect to the original DLN in Equation~(\ref{eq:dln_setup}) with $\mathcal{A} = Id$ and with respect to the compressed network defined in Equation~(\ref{eq:comp_deep_mf}). Then, if  Assumption~\ref{ass:incremental} holds such that $c_{\text{vec}} = 0$, we have that $\forall t\geq 0$, 
    \begin{align}
    \label{eq:recovery_error_app}
        \|\bm{W}_{L:1}(t) - \bM^*\|^2_{\mathsf{F}} \geq \|\widetilde{\bm{W}}_{L:1}(t) - \bM^*\|^2_{\mathsf{F}}.
    \end{align}
\end{manualtheoreminner}

\begin{proof} For clarity, we organize the proof into subsections.

\paragraph{Notation \& Assumptions.}

By Assumption~\ref{ass:incremental}, let us define a sequence of time points $t_1^{\text{orig}} \leq t_2^{\text{orig}} \leq \ldots \leq t_r^{\text{orig}}$ and $t_1^{\text{comp}} \leq t_2^{\text{comp}} \leq \ldots \leq t_r^{\text{comp}}$ in which the $i$-th principal components are fitted up to some precision defined by  $c^{\text{orig}}_{\text{val}}, c^{\text{orig}}_{\text{vec}}$ and $c^{\text{comp}}_{\text{val}}, c^{\text{comp}}_{\text{vec}}$ for both the original and compressed network, respectively. Recall that $t_i$ denotes the time in which $i$-th principal components are fitted and when the ($i+1$)-th components start being learned. We assume that $c_{\text{vec}}^{\text{orig}} = 0$ and that for all $t < t_i$, 
    \begin{align*}
    \sigma_{i+1}(t) = \widetilde{\sigma}_{i+1}(t) = \epsilon^L,
\end{align*}
where $\sigma_{i}(t)$ and $\widetilde{\sigma}_{i}(t)$ denotes the $i$-th singular value of the original and compressed network, respectively. Notice that by Theorem~\ref{thm:parsimony}, we have that $\widetilde{\bm u}_i(t) = \bm u_i^*$ and $\widetilde{\bm v}_i(t) = \bm v_i^*$ for all $t$, where $\tilde{\bm u}_i(t)$ and $\widetilde{\bm v}_i(t)$ denote the $i$-th left and right singular vector of $\widetilde{\bm{W}}_{L:1}(t)$ respectively, and so $c_{\text{vec}}^{\text{comp}} = 0$ by definition. 
The second assumption states that the singular values are learned incrementally, and that the trailing $(i+1)$ to $d$ singular values stay at initialization until all the first $i$ singular values are learned. 

\paragraph{Roadmap of Proof.}  

Notice that by using Lemma~\ref{lem:von}, we can write the inequality in Equation~(\ref{eq:recovery_error_app}) as 
    \begin{align*}
        \|\bm{W}_{L:1}(t) - \bM^*\|^2_{\mathsf{F}} \geq \sum_{i=1}^{\hat{r}} (\sigma_i(t) - \sigma_i^*)^2 \underset{?}{\geq}  \sum_{i=1}^{\hat{r}} (\widetilde{\sigma}_i(t) - \sigma_i^*)^2 =   \|\widetilde{\bm{W}}_{L:1}(t) - \bM^*\|^2_{\mathsf{F}}.
    \end{align*}
Then, to establish the result, we will show that 
\begin{align*}
    \sum_{i=1}^{\hat{r}} (\sigma_i(t) - \sigma_i^*)^2 \geq \sum_{i=1}^{\hat{r}} (\widetilde{\sigma}_i(t) - \sigma_i^*)^2.
\end{align*}
To do this, we will first show that $t_i^{\text{comp}} \leq t_i^{\text{orig}}$ for all $i \in [r]$ using mathematical induction, as that would imply that the singular values of the compressed network are learned faster, leading to a lower recovery error. Then, we will show that the inequality also holds for all $j \in [r+1, \hat{r}]$.

    \paragraph{Base Case.} Let $t_0^{\text{orig}}$ and $ t_0^{\text{comp}}$ denote the initial time points. For the base case, the initialization gives us 
    $t_0^{\text{orig}} = t_0^{\text{comp}} = 0$.
    

\paragraph{Inductive Hypothesis.} By the inductive hypothesis, suppose that $t_i^{\text{orig}} \geq t_i^{\text{comp}}$ such that the $i$-th principal components of the compressed network are fitted faster than the original network. 

\paragraph{Inductive Step (Top-$r$ Components).} For the inductive step, we need to show $t_{i+1}^{\text{orig}} \geq t_{i+1}^{\text{comp}}$. We will do this by analyzing the dynamics of the singular values of both networks.

Notice that by the inductive hypothesis and by assumption, for all $t < t_{i}^{\text{comp}} \leq t_{i}^{\text{orig}}$, we have
\begin{align}
\label{eq:sval_first}
    \sigma_{i+1}(t) = \widetilde{\sigma}_{i+1}(t) = \epsilon^L.
\end{align}
Then, for all $t > t_{i}^{\text{orig}} \geq t_{i}^{\text{comp}}$, by~\Cref{thm:arora}, 
the gradient dynamics of $\sigma_{i+1}(t)$ are governed by
\begin{align*}
\dot{\sigma}_{i+1}(t) 
&= -L \cdot (\sigma^2_{i+1}(t))^{1-1/L} \cdot \langle \nabla_{\bm{W}_{L:1}}\ell(\bm{\Theta}(t)), \bm{u}_{i+1}(t)\bm{v}_{i+1}^{\top}(t)  \rangle, 
\\
&= -L \cdot (\sigma^2_{i+1}(t))^{1-1/L} \cdot (\bm{u}_{i+1}^\top(t)\bm{W}_{L:1}(t)\bm{v}_{i+1}(t) - \bm{u}_{i+1}^\top(t)\bm{M}^*\bm{v}_{i+1}(t))
\\
&= -L \cdot (\sigma^2_{i+1}(t))^{1-1/L} \cdot (\sigma_{i+1}(t) - \bm{u}_{i+1}^\top(t)\bm{M}^*\bm{v}_{i+1}(t))
\\
&\leq -L \cdot (\sigma^2_{i+1}(t))^{1-1/L} \cdot (\sigma_{i+1}(t) - \sigma_{i+1}^*),
\end{align*}
where the last inequality follows from that $c_{\text{vec}}^{\text{orig}} = 0$ and so $\bm{u}_{i+1}^\top(t)$ and $\bm{v}_{i+1}^\top(t)$ are vectors that are orthogonal to $\text{span}\{\bm{u}_j^*,\, 1\leq j\leq i\}$ and $\text{span}\{\bm{v}_j^*,\, 1\leq j\leq i\}$, such that $\bm{u}_{i+1}^\top(t)\bm{M}^*\bm{v}_{i+1}(t)$ is at most $\sigma_{i+1}^*$.


\noindent Similarly, notice that for the compressed network, we have for times $t > t_{i}^{\text{comp}}$,
\begin{align}
\label{eq:comp_sval_dynamic}
\dot{\widetilde{\sigma}}_{i+1}(t) 
&= -L \cdot (\widetilde{\sigma}^2_{i+1}(t))^{1-1/L} \cdot \langle \nabla_{\widetilde{\bm{W}}_{L:1}}\ell(\widetilde{\bm{\Theta}}(t)), \bm{u}^*_{i+1}\bm{v}_{i+1}^{*\top}  \rangle,
\\
&= -L \cdot (\sigma^2_{i+1}(t))^{1-1/L} \cdot (\bm{u}_{i+1}^{*\top}\widetilde{\bm{W}}_{L:1}(t)\bm{v}^*_{i+1} - \bm{u}_{i+1}^{*\top}\bm{M}^*\bm{v}^*_{i+1})
\\
&= -L \cdot (\widetilde{\sigma}^2_{i+1}(t))^{1-1/L} \cdot (\widetilde{\sigma}_{i+1}(t) - \sigma_{i+1}^*).
\end{align}
Therefore by the comparison theorem for ODEs~(Theorem 1.3 of~\cite{Teschl2012OrdinaryDE}) we have that
\begin{align}
\label{eq:sval_last}
    \epsilon^L \leq \sigma_{i+1}(t) \leq \widetilde{\sigma}_{i+1}(t) \leq \sigma_{i+1}^*, \quad\,\,\,  \forall t > t_{i}^{\text{orig}},
\end{align}
Then, for all $t$ in $t_{i}^{\text{comp}} \leq t \leq t_{i}^{\text{orig}}$, the singular values of the  compressed network follow the dynamics outlined in Equation~(\ref{eq:comp_sval_dynamic}), whereas the singular value of the original network stays at initialization by assumption. Thus,
\begin{align}
\label{eq:sval_middle}
\epsilon^L = \sigma_{i+1}(t) \leq \widetilde{\sigma}_{i+1}(t) \leq \sigma_{i+1}^*.
\end{align}
Hence, by combining \Cref{eq:sval_first},~\Cref{eq:sval_last},~\Cref{eq:sval_middle}, we have that for  $t_{i+1}^{\text{orig}} \geq t_{i+1}^{\text{comp}}$ for all $i \in [r]$.

\begin{figure}[t]
    \centering
    \includegraphics[width=0.5\textwidth]{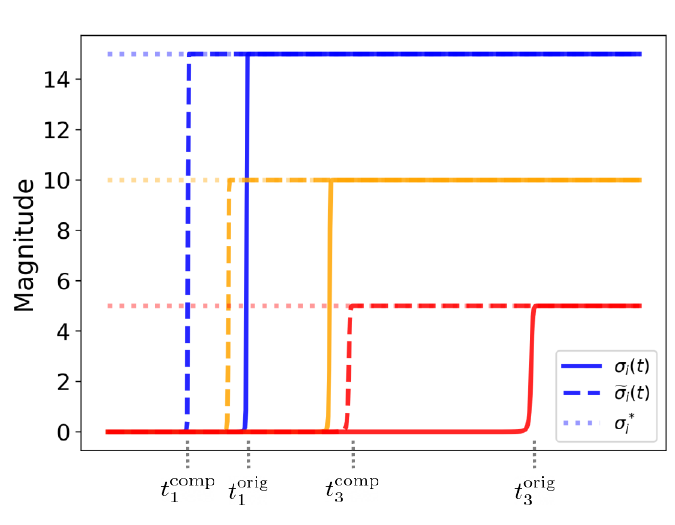}
    \caption{\textbf{Visual illustration of the proof technique for Theorem~\ref{thm:recovery_mf}. } The proof involves showing that $t_i^{\text{comp}} \leq t_i^{\text{orig}}$, which amounts to showing that the singular values of the compressed network fits the target value faster than the original network.}
    \label{fig:proof_helper}
\end{figure}

\paragraph{Inductive Step (Residual Components).}
For the remaining residual singular values, 
notice that $\forall t > t_r^{\text{orig}}$, the dynamics of the residual singular values $\forall j \in [r+1, \hat{r}]$ are governed by 
\begin{align*}
\dot{\sigma}_{j}(t) 
&= -L \cdot (\sigma^2_{j}(t))^{1-1/L} \cdot \langle \nabla_{\bm{W}_{L:1}}\ell(\bm{\Theta}(t)), \bm{u}_{j}(t)\bm{v}_{j}^{\top}(t)  \rangle, 
\\
&= -L \cdot (\sigma^2_{j}(t))^{1-1/L} \cdot (\bm{u}_{j}^\top(t)\bm{W}_{L:1}(t)\bm{v}_{j}(t) - \bm{u}_{j}^\top(t)\bm{M}^*\bm{v}_{j}(t))
\\
&= -L \cdot (\sigma^2_{j}(t))^{1-1/L} \cdot (\sigma_{j}(t) - \bm{u}_{j}^\top(t)\bm{M}^*\bm{v}_{j}(t))
\\
&= -L \cdot (\sigma^2_{j}(t))^{1-1/L} \cdot \sigma_{j}(t),
\end{align*}
which is identical to that of the compressed network, and so $t_{i+1}^{\text{orig}} \geq t_{i+1}^{\text{comp}}$ for all $i \in [r+1, \hat{r}]$.
\noindent Thus,
\begin{align}
    (\sigma_i(t) - \sigma_i^{*})^2 \geq (\widetilde{\sigma}_i(t) - \sigma_j^{*})^2, \quad\,\,\,   \forall i \in [\hat{r}],
\end{align}
and so we have
    \begin{align*}
        \|\bm{W}_{L:1}(t) - \bm{M}^*\|^2_{\mathsf{F}}
        \geq \sum_{i=1}^{\hat{r}} (\sigma_i(t) - \sigma_i^{*})^2 \geq \sum_{i=1}^{\hat{r}} (\widetilde{\sigma}_i(t) - \sigma_i^{*})^2 
        = \|\widetilde{\bm{W}}_{L:1}(t) - \bm{M}^*\|^2_{\mathsf{F}},
    \end{align*}
which completes the proof.
\end{proof}

\paragraph{Remarks.} In Figure~\ref{fig:proof_helper}, we provide a visual illustration of the proof technique. Here, we can clearly see that $t_i^{\text{comp}} \leq t_i^{\text{orig}}$, which leads to the lower recovery error at the same time instance $t$. To rigorously show that $t_i^{\text{comp}} \leq t_i^{\text{orig}}$, we relied on using the gradient flow result stated in~\Cref{thm:arora} and the incremental learning assumption presented in Assumption~\ref{ass:incremental}. We justified the use of Assumption~\ref{ass:incremental} in the beginning of Appendix~\ref{sec:incremental_app}, where we showed that the principal components of the DLN are fitted incrementally, and that generally, $c_{\text{vec}} = c_{\text{val}} = 0$ (see Figure~\ref{fig:ass_cvals}).

\subsection{Auxiliary Results}

\begin{lemma}
\label{lem:von}
    Consider two matrices $\bm{A}, \bm{B} \in \R^{d_1 \times d_2}$ and their respective compact singular value decompositions (SVD) $\bm{A} = \bm{U}_A \bm{\Sigma}_A \bm{V}_A^{\top}$ and $\bm{B} = \bm{U}_B \bm{\Sigma}_B \bm{V}_B^{\top}$. We have that
    \begin{align*}
        \|\bm{A}-\bm{B}\|_{\mathsf{F}}^2 \geq \|\bm{\Sigma}_A-\bm{\Sigma}_B\|_{\mathsf{F}}^2.
    \end{align*}
\end{lemma}
\begin{proof}
    By using the definition of the Frobenius norm, we have
    \begin{align*}
        \|\bm{A}-\bm{B}\|_{\mathsf{F}}^2 &= \|\bm{U}_A \bm{\Sigma}_A \bm{V}_A^{\top} - \bm{U}_B \bm{\Sigma}_B \bm{V}_B^{\top}\|_{\mathsf{F}}^2 \\
        &= \tr(\bm{\Sigma}_A^2)  -2\tr(\bm{A}^{\top}\bm{B}) + \tr(\bm{\Sigma}_B^2) \\
        &\geq \tr(\bm{\Sigma}_A^2)  -2\tr(\bm{\Sigma}_A\bm{\Sigma}_B) + \tr(\bm{\Sigma}_B^2) \tag{Von-Neumann Inequality} \\
        &= \|\bm{\Sigma}_A-\bm{\Sigma}_B\|_{\mathsf{F}}^2.
    \end{align*}
This proves the desired result.
\end{proof}

\clearpage